\documentclass[runningheads,envcountsame]{llncs}

\usepackage{etoolbox}
\newtoggle{extendedversion}
\togglefalse{extendedversion}

\usepackage[hidelinks]{hyperref}

\spnewtheorem{definition}{Definition}{\bfseries}{\itshape}
\spnewtheorem{example}{Example}{\itshape}{\rmfamily}
\usepackage{graphicx}
\usepackage{tikz}
\usepackage{pgfplots}
\usepgfplotslibrary{statistics}
\usepackage{amsmath,nccmath,bm}
\usepackage{mathrsfs}
\usepackage{amsfonts}
\usepackage{amssymb}
\usepackage[multiple]{footmisc}
\usepackage{algorithm,algorithmicx}
\usepackage[noend]{algpseudocode}
\newcommand*\Let[2]{\State #1 $\gets$ #2}
\algrenewcommand\algorithmicrequire{\textbf{Precondition:}}
\algrenewcommand\algorithmicensure{\textbf{Postcondition:}}

\spnewtheorem{myexp}{Example}{\itshape}{\rmfamily}
\newcounter{running}
\newenvironment{myexpcont}
{\begin{myexp}\hspace{-0.5em}{\it \arabic{running}.}\hspace{0.5em}}
  {{\unskip\nobreak\hfil\penalty50\hskip1em\null\nobreak\hfil$\lozenge$\parfillskip=0pt\finalhyphendemerits=0\endgraf}\end{myexp}\addtocounter{myexp}{-1}\addtocounter{running}{1}}

\DeclareMathOperator{\Posi}{Posi}

\DeclareMathOperator{\Rs}{Rs}
\DeclareMathOperator{\Ex}{Ex}

\newcommand{\R}{\mathbb{R}}

\newcommand{\N}{\mathbb{N}}
\newcommand{\V}{\mathscr{V}}
\newcommand{\Q}{\mathscr{Q}}
\newcommand{\QQ}{\mathscr{A}}

\newcommand{\Ka}{\bar{\mathbf{K}}}
\newcommand{\A}{\mathcal{A}}
\newcommand{\AAA}{A_1,\dots,A_n}

\newcommand{\Lamp}{\Lambda^+}
\newcommand{\lala}{\boldsymbol{\lambda}}
\newcommand{\mumu}{\boldsymbol{\mu}}

\newcommand{\X}{\mathcal{X}}

\newcommand{\succc}{>}

\newcommand{\precc}{<}
\newcommand{\precceq}{\le}

\usepackage{mathtools}

\newcommand{\uu}{\mathbf{u}}
\newcommand{\vv}{\mathbf{v}}

\usepackage[capitalise,noabbrev]{cleveref}

\newenvironment{proofof}[1]{\par\noindent{\bfseries\upshape Proof of #1\ }}{\qed\medskip}

\begin{document}
\title{Inference with Choice Functions Made Practical}
\author{Arne Decadt 
\and
Jasper De Bock
\and
Gert de Cooman
}
\authorrunning{A. Decadt et al.}
\institute{
FLip, ELIS, Ghent University, Belgium\\
\email{\{arne.decadt,jasper.debock,gert.decooman\}@ugent.be}}
\maketitle             
\begin{abstract}
We study how to infer new choices from previous choices in a conservative manner.
To make such inferences, we use the theory of choice functions: a unifying mathematical framework for conservative decision making that allows one to impose axioms directly on the represented decisions.
We here adopt the coherence axioms of De Bock and De Cooman (2019).
We show how to naturally extend any given choice assessment to such a coherent choice function, whenever possible, and use this natural extension to make new choices.
We present a practical algorithm to compute this natural extension and provide several methods that can be used to improve its scalability.

\end{abstract}
\section{Introduction}
\label{sec:intro}
In classical probability theory, decisions are typically made by maximizing expected utility.
This leads to a single optimal decision, or a set of optimal decisions all of which are equivalent.
In imprecise probability theory, where probabilities are only partially specified, this decision rule can be generalized in multiple ways;
Troffaes \cite{troffaes2007decision} provides a nice overview.
A typical feature of the resulting imprecise decision rules is that they will not always yield a single optimal decision, as a decision that is optimal in one probability model may for example be sub-optimal in another.

We will not focus on one particular imprecise decision rule though.
Instead, we will adopt the theory of choice functions: a mathematical framework for decision making that incorporates several (imprecise) decision rules as special cases, including the classical approach of maximizing expected utility \cite{de2020archimedean,de2019interpreting,seidenfeld2010coherent}.
An important feature of this framework of choice functions is that it allows one to impose axioms directly on the decisions that are represented by such a choice function \cite{de2019interpreting,seidenfeld2010coherent,van2018natural}.
We here adopt the coherence axioms that were put forward by De Bock and De Cooman \cite{de2019interpreting}.

As we will explain and demonstrate in this contribution, these coherence axioms can be used to infer new choices from previous choices.
For any given assessment of previous choices that is compatible with coherence, we will achieve this by introducing the so-called natural extension of this assessment: the unique most conservative coherent choice function that is compatible with the assessment.

We start in \cref{sec:choicef} with an introduction to choice functions and coherence and then go on to define the natural extension of choice assessments in \cref{sec:Natext}.
From then on, we work towards a method to compute this natural extension.
An important step towards this goal consists in translating our problem to the setting of sets of desirable option sets; we do this in \cref{sec:SODOS}.
In this setting, as we show in \cref{sec:choose}, the problem takes on a more manageable form.
Still, the complexity of the problem depends rather heavily on the size of the assessment that is provided.
To address this issue, \cref{sec:Ass} presents several methods that can be used to replace an assessment by an equivalent yet smaller one.
\cref{sec:prac} then presents a practical algorithm that implements our approach.
\cref{sec:conc} concludes the paper and provides some suggestions for future work.

\iftoggle{extendedversion}{
To adhere to the page-limit constraint, all proofs are omitted.
They are available in an extended on-line version~\cite{decadt2020inference}.
}{}

\section{Choice functions}\label{sec:choicef}
A choice function $C$ is a function that, when applied to a set of options, chooses one or more options from that set.
Usually the options are actions that have a corresponding reward.
This reward furthermore depends on the state of an unknown---uncertain---variable $X$ that takes values in a set $\X$.
We will assume that the rewards can be represented by real numbers, on a real-valued utility scale.
In this setting, an option $u$ is thus a function from states $x$ in $\X$ to $\R$.
We will denote the set of all possible options by $\V$.
Evidently, $\V$ is a vector space with vector addition $u+v\colon x\mapsto u(x)+v(x)$ and scalar multiplication $\lambda u\colon x\mapsto \lambda u(x)$ for all $x\in \X$, $\lambda\in\R$ and $u,v\in\V$.
Moreover, we endow $\V$ with the partial order $\precceq$ defined by $u\precceq v\Leftrightarrow (\forall x\in \X) u(x)\le v(x)$ for all $u,v\in\V$ and a corresponding strict version $\precc$, where $u\precc v$ if both $u\neq v$ and $u\precceq v$.\footnote{Our results in section 2 and 3 also work for any ordered vector space over the real numbers; we restrict ourselves to the particular order $\precc$ for didactic purposes.}
To make this more tangible, we consider the following toy problem as a running example.
\begin{myexpcont}
A farming company cultivates tomatoes and they have obtained a large order from a foreign client.
However, due to government regulations they are not sure whether they can deliver this order.
So, the state space $\X$ is $\{$order can be delivered, order cannot be delivered$\}$.
The company now has multiple options to distribute their workforce.
They can fully prepare the order, partially prepare the order or not prepare the order at all.
Since $\X$ only has two elements, we can identify the options with vectors in $\R^2$.
We will let the first component of these vectors correspond to the reward if the order can be delivered.
For example, the option of fully preparing the order could correspond to the vector $v_1\coloneqq(5,-3)$.
If the order goes through, then the company receives a payment---or utility---of $5$ for that order.
However, if the order does not go through, the company ``receives'' a negative reward $-3$, reflecting the large amount of resources that they spent on an order that could not be delivered in the end.
\end{myexpcont}

We will restrict ourselves to choices from finite sets of options. That is, the domain of our choice functions will be $\Q\coloneqq\{A\subset\V:n\in\N,|A|=n\}\cup \{\emptyset\}$: the set of all finite subsets of $\V$, including the empty set.
Formally, a choice function is then any function $C\colon \Q\to\Q$ such that $C(A)\subseteq A$ for all $A\in\Q$.
We will also consider the corresponding rejection function $R_C\colon \Q\to\Q\colon A\mapsto A\setminus C(A)$. 

\begin{myexpcont}
We will let the choice function $C$ correspond to choices that the strategic advisor of the company makes or would make for a given set of options, where these choices can be multi-valued whenever he does not choose a single option. 
Suppose for example that he has rejected $v_3$ and $v_4$ from a set $B_1\coloneqq\{v_1,v_2,v_3,v_4\}$, with
$
v_1\coloneqq(5,-3),$ $v_2\coloneqq(3,-2),$ $v_3\coloneqq(1,-1),$ and $v_4\coloneqq(-2,1),
$
but remains undecided about whether to choose $v_1$ or $v_2$.
This corresponds to the statement $C(B_1)=\{v_1,v_2\}$, or equivalently, $R_C(B_1)=\{v_3,v_4\}$.
\end{myexpcont}

We will give the following interpretation to these choice functions.
For every set $A\in \Q$ and option $u\in A$, we take $u\in C(A)$---$u$ is chosen---to mean that there is no other option in $A$ that is preferred to $u$.
Equivalently, $u\in R_C(A)$---$u$ is rejected from $A$---if there is an option in $A$ that is preferred to $u$.
The preferences in this interpretation are furthermore taken to correspond to a strict partial vector order $\prec$ on $\V$ that extends the original strict order $<$.
This implies that it should have the following properties: for all $u,v,w\in\V$ and $\lambda>0$, 
\begin{flalign*}
&O_1. \quad\text{$u\not\prec v$ or $v\not\prec u$}, &\text{(antisymmetry)}&\\
&O_2. \quad\text{if $u\prec v$ and $v\prec w$ then also $u\prec w$}, &\text{(transitivity)}&\\
&O_3. \quad\text{if $u\prec v$ then also $u+w\prec v+w$}, &\text{(translation invariance)}&\\
&O_4. \quad\text{if $u\prec v$ then also $\lambda u\prec \lambda v$}. &\text{(scaling invariance)}
\end{flalign*}
Crucially, however, the strict partial order $\prec$ need not be known.
Instead, in its full generality, our interpretation allows for the use of a set of strict partial orders, only one of which is the true order $\prec$.
As shown in reference \cite{de2019interpreting}, this interpretation can be completely characterized using five axioms for choice functions.
Rather than simply state them, we will motivate them one by one starting form our interpretation.

The first axiom states that we should always choose at least one option, unless we are choosing from the empty set:
\begin{flalign*}
\mathrm{C}_0. \quad C(A)\neq \emptyset \text{ for all $A\in\Q\setminus\{\emptyset\}$}. &&
\end{flalign*}
This follows directly from our interpretation. 
Indeed, if every option in $A$ would be rejected, then for every option $v$ in $A$, there would be some other option in $A$ that is preferred to $v$.
Transitivity $(O_2)$ would then imply that $A$ contains an option that is preferred to itself, contradicting antisymmetry $(O_1)$.
To understand the second axiom, first observe that it follows from translation invariance $(O_3)$ that $u\prec v$ if and only if $0\prec v-u$.
So if we let $A-u\coloneqq \{v-u\colon v\in A\}$,
then it follows from our interpretation that $u$ is chosen from $A$ if and only if $0$ is chosen from $A-u$:
\begin{flalign*}
\mathrm{C}_1. \quad u\in C(A)\Leftrightarrow 0\in C(A-u)\text{ for all $A\in\Q$ and $u\in A$}. &&
\end{flalign*}
An important consequence of this axiom, is that knowing from which sets zero is chosen suffices to know the whole choice function.
To formalize this, we introduce for any choice function $C$ a corresponding set of option sets
\begin{equation}\label{eq:KC}
K_C\coloneqq\{A\in \Q \colon 0\notin C(A\cup\{0\})\}.
\end{equation}
For any option set $A\in K_C$, it follows from our interpretation for $C$ that $A\cup\{0\}$ contains at least one option that is preferred to zero.
Or equivalently, since zero is not preferred to itself because of $O_1$, that $A$ contains at least one option that is preferred to zero.
We call such an option set $A$ a \emph{desirable option set} and will therefore refer to $K_C$ as the \emph{set of desirable option sets} that corresponds to $C$.
Since it follows from $\mathrm{C}_1$ that
\begin{equation}\label{eq:goback}
(\forall A \in \Q)(\forall u\in A)\; u\not\in C(A) \Leftrightarrow A-u\in K_C,
\end{equation}
We see that the set of desirable option sets $K_C$ fully characterizes the choice function $C$.
Whenever convenient, we can therefore express axioms for $C$ in terms of $K_C$ as well.
The next axiom is a first example where this is convenient.
Since the preference order~$\prec$ is taken to extend the order~$<$, it follows that we must prefer all elements of $\V_{>0}\coloneqq\{u\in\V\colon 0< u\}$ to zero.
Hence, $\{u\}$ is a desirable option set for all $u\in \V_{>0}$:
\begin{flalign*}
\mathrm{C}_2. \quad \{u\}\in K_C\text{ for all }u\in\V_{>0}. &&
\end{flalign*}
Or to state it in yet another way: the set of positive singletons $\V^s_{>0}\coloneqq\{\{u\}\colon u\in \V_{>0}\}$ should be a subset of $K_C$.

Another axiom that is easier to state in terms of $K_C$ follows from the fact that we can take positive linear combinations of preferences. 
For example, if we have two non-negative real numbers $\lambda_1,\lambda_2$ and they are not both zero, and we know that $0\prec u_1$ and $0\prec u_2$, then it follows from $O_2, O_3$ and $O_4$ that also $0\prec \lambda_1 u_1 + \lambda_2 u_2$.
To state this more compactly, we let\footnote{We will denote tuples in boldface and their elements roman with a positive integer index corresponding to their position.}
\[
\textstyle
\R^{n,+}\coloneqq \{\lala\in\R^n:(\forall j\in \{1,\dots,n\})(\lambda_j\geqslant0),\sum_{j=1}^n \lambda_j>0\}
\]
for any positive integer $n$ and introduce a product operation $\lala \uu\coloneqq \sum_{j=1}^n \lambda_j u_j$ for tuples $\lala=(\lambda_1,\ldots,\lambda_n)\in\R^{n}$ with tuples of vectors $\uu=(u_1,\ldots,u_n)\in\mathscr{V}^n$.
Then for any $\lala\in \R^{2,+}$ and any $\uu\in\V^2$ such that $0\prec u_1$ and $0\prec u_2$, we have that $0\prec \lala \uu$.
Consider now two sets $A,B\in K_C$. 
Then as we explained after \cref{eq:KC}, each of them contains at least one option that is preferred to zero.
Hence, by the reasoning above, there will be at least one pair of options $\uu=(u_1,u_2)$ in $A\times B$ for which $0\prec u_1$ and $0\prec u_2$ and thus $0\prec \lala \uu=\lambda_1 u_1+\lambda_2 u_2$ for $\lala\in\R^{2,+}$.
Thus, the set of all such possible combinations---where $\lala$ can depend on $\uu$---will contain at least one option preferred to zero en must therefore belong to $K_C$:
\begin{flalign*}
\mathrm{C}_3. \quad \{\boldsymbol{\lambda}(\mathbf{u}) \mathbf{u}\colon \mathbf{u} \in A\times B\}\in K_C \text{ for all } \boldsymbol{\lambda}\colon A\times B \to \R^{2,+} \text{ and } A,B\in K_C. &&
\end{flalign*}

The final axiom states that if an option $u$ is rejected from an option set $A$, then it will also be rejected from any superset $B$:
\begin{flalign*}
\mathrm{C}_4. \quad A\subseteq B \Rightarrow R_C(A)\subseteq R_C(B) \text{ for all }A,B\in \Q. &&
\end{flalign*}
Once more, this follows from our interpretation for $C$.
If $u$ is rejected from $A$, then this means that there is an element $v\in A$ that is preferred to $u$. 
Since $v$ and $u$ also belong to $B$, it thus follows that $u$ is rejected from $B$ as well.
This axiom is also known as Sen's condition $\alpha$ \cite{sen1971choice}.

\begin{definition}
We call a choice function $C\colon\Q\to\Q$ \emph{coherent} if it satisfies axioms $\mathrm{C}_0-\mathrm{C}_4$. The set of all coherent choice functions is denoted by $\mathcal{C}$.
\end{definition}
A crucial point is that the axioms $\mathrm{C}_0-\mathrm{C}_4$ are the same as the axioms $\mathrm{R}_0-\mathrm{R}_4$ in \cite{de2019interpreting}, but tailored to our notation and in a different order; $\mathrm{C}_0$ corresponds to $\mathrm{R}_1$ whereas $\mathrm{C}_1$ corresponds to $\mathrm{R}_0$.
Interestingly, as proven in reference \cite{de2019interpreting}, the axioms $\mathrm{C}_0-\mathrm{C}_4$ are therefore not only necessary for a choice function to correspond to a set of strict partial orders, but sufficient as well.

\section{Natural extension}\label{sec:Natext}

Fully specifying a coherent choice function is hard to do in practice, because this would amount to specifying a set-valued choice for every finite set of options, while at the same time taking into account the coherence axioms.
Instead, a user will typically only be able to specify $C(A)$---or $R_C(A)$---for option sets $A$ in some small---and often finite---subset $\mathcal{O}$ of $\Q$.
For all other option sets $A\in \Q\setminus \mathcal{O}$, we can then set $C(A)=A$ because this adds no information to the choice function assessment.
The resulting choice function may not be coherent though.
To make this more concrete, let us go back to the example.

\begin{myexpcont}
Suppose that the strategic advisor of the farming company has previously rejected the options $v_3$ and $v_4$ from the option set $B_1$, as in {Example~1.2}, and has chosen $v_6$ from $B_2\coloneqq\{v_5,v_6\}$, where
$
v_5\coloneqq(3,1)$ and $v_6\coloneqq(-4,8).
$
This corresponds to the choice function assessments
\(
C(B_1)=\{v_1,v_2\} \text{ and } C(B_2)=\{v_6\}.
\)
Suppose now that the company's strategic advisor has fallen ill and the company is faced with a new decision problem that amounts to choosing from the set $B_3=\{(-3,4),(0,1),(4,-3)\}$.
Since no such choice was made before, the conservative option is to make the completely uninformative statement $C(B_3)=B_3$.
However, perhaps the company can make a more informative choice by taking into account the advisor's previous choices?
\end{myexpcont}

In order to make new choices based on choices that have already been specified---as in the last question of the example---we can make use of coherence.
Indeed, all the coherence axioms except $C_0$ allow one to infer new choices---or rejections---from other ones.
In this way, we obtain a new choice function $C'$ that is more informative than $C$ in the sense that $C'(A)\subseteq C(A)$ for all $A\in \Q$.
Any such choice function $C'$ that is more informative than $C$, we call an extension of $C$.
In order to adhere to the coherence axioms, we are looking for an extension $C'$ of $C$ that is coherent.
We denote the set of all such extensions by $\mathcal{C}$.

Whenever at least one such coherent extension exists, we call $C$ consistent.
The least informative such coherent extension of $C$ is then called its natural extension, as it is the only coherent extension of $C$ that follows solely from $C$ and the coherence axioms, without adding any additional unwarranted information.
In order to formalize this notion, we let
\[
\mathcal{C}_C\coloneqq \{C'\in \mathcal{C} \colon C'(A)\subseteq C(A) \text{ for all } A\in\Q \}
\] 
and let $\Ex(C)$ be defined by
\[
\Ex(C)(A)\coloneqq \bigcup_{C'\in \mathcal{C}_C} C'(A) \text{ for all $A\in\Q$},
\]
where, by convention, the union over an empty set is taken to be empty itself.

\begin{definition}
For any choice function $C$, we call $C$ \emph{consistent} if $\mathcal{C}_C\neq \emptyset$ and then refer to $\Ex(C)$ as the {\it natural extension} of $C$.
\end{definition}
\begin{theorem}\label{th:coh}
For any choice function $C$ that is consistent, $\Ex(C)$ is the least informative coherent extension of $C$. That is, $\Ex(C)\in\mathcal{C}_C$ and, for all $C'\in\mathcal{C}_C$, we have that $C'(A)\subseteq \Ex(C)(A)$ for all $A\in\Q$. If on the other hand $C$ is not consistent, then $\Ex(C)$ is incoherent and $\Ex(C)(A)=\emptyset$ for all $A\in \Q$.
\end{theorem}
Given a choice function $C$ that summarizes earlier choices and a new decision problem that consists in choosing from some set of options $A$, we now have a principled method to make this choice, using only coherence and the information present in $C$.
First we should check if $C$ is consistent.
If it is not, then our earlier choices are not compatible with coherence, and would therefore better be reconsidered.
If $C$ is consistent, then the sensible choices to make are the options in $\Ex(C)(A)$, since the other options in $A$ can safely be rejected taking into account coherence and the information in $C$.
If $\Ex(C)(A)$ contains only one option, we arrive at a unique optimal choice.
If  not, then adding additional information is needed in order to be able to reject more options.
The problem with this approach, however, is that it requires us to check consistency and evaluate $\Ex(C)(A)$.
The rest of this contribution is devoted to the development of practical methods for doing so.
We start by translating these problems to the language of sets of desirable option sets.

\section{Sets of Desirable Option Sets}\label{sec:SODOS}
As explained in \cref{sec:choicef}, every coherent choice function $C$ is completely determined by its corresponding set of desirable option sets $K_C$.
Conversely, with any set of desirable option sets $K\subseteq \Q$, we can associate a choice function $C_K$, defined by
\begin{equation}\label{eq:CK}
C_K(A)\coloneqq \{u\in A \colon (A-u)\setminus \{0\}\notin K\}\text{ for all $A\in\Q$.}
\end{equation}
In order for $C_K$ to be coherent, it suffices for $K$ to be coherent in the following sense.

\begin{definition}\label{def:CohK}
A set of desirable option sets $K\subseteq \Q$ is called \emph{coherent} \cite{de2018desirability} if for all $A,B\in K$:  
\begin{flalign*}
&\mathrm{K}_0. \quad A\setminus\{0\}\in K;&&\\
&\mathrm{K}_1. \quad \{0\}\notin K;&&\\
&\mathrm{K}_2. \quad \V^s_{>0} \subseteq K;&&\\
&\mathrm{K}_3. \quad \{\boldsymbol{\lambda}(\mathbf{u}) \mathbf{u}\colon \mathbf{u} \in A\times B\} \subseteq K \text{ for all }\boldsymbol{\lambda}\colon A\times B \to \R^{2,+};&&\\
&\mathrm{K}_4. \quad A \cup Q \in K \text{ for all }Q\in\Q.&&
\end{flalign*}
We denote the set of all coherent $K$ by $\Ka$.
\end{definition}
\begin{proposition}\label{prop:KcohCK}
If a set of desirable option sets $K\subseteq \Q$ is coherent, then $C_K$ is coherent too.
\end{proposition}
\begin{proposition}\label{prop:CcohKC}
If a choice function $C$ is coherent, then $K_C$ is coherent too.
\end{proposition}
\begin{theorem}\label{th:unique}
The map $\Phi\colon \mathcal{C} \to \Ka\colon C\mapsto K_C$ is a bijection and has inverse $\Phi^{-1}(K)=C_K$ for all $K\in\Ka$.
\end{theorem}
In other words: every coherent choice function $C$ corresponds uniquely to a coherent set of desirable option sets $K$, and vice versa.

The plan is now to use this connection to transform the problem of computing the natural extension of a choice function $C$ to a similar problem for sets of desirable option sets.
We start by transforming the choice function $C$ into a set of option sets.
One way to do this would be to consider the set of desirable option sets $K_C$.
However, there is also a smarter way to approach this that yields a more compact representation of the information in $C$.
\begin{definition}\label{def:assessment}
An \emph{assessment} is any subset of $\Q$.
We denote the set of all assessments, the power set of $\Q$, by $\QQ$.
In particular, with any choice function $C$, we associate the assessment
\[
\A_{C}\coloneqq\{C(A)-u\colon A\in\Q,u\in R_C(A)\}.
\]
\end{definition}

\begin{myexpcont}
In our running example, the assessment that corresponds to $C$ is $\A_C=\{A_1,A_2,A_3\}$, with
\(
A_1\coloneqq\{w_1,w_2\}\), \(A_2\coloneqq\{w_3,w_4\}\) and \( A_3\coloneqq\{w_5\},
\) 
where
\(
w_1\coloneqq v_1-v_3=(4,-2)
\) and similarly
\(
w_2\coloneqq v_2-v_3=(2,-1),
\)
$w_3\coloneqq v_1-v_4=(7,-4)$,
$w_4\coloneqq v_2-v_4=(5,-3)$ and $w_5\coloneqq v_6-v_5=(-7,7)$.
\end{myexpcont}

An assessment such as $\A_C$ may not be coherent though.
To extend it to a coherent set of desirable option sets, we will use the notions of consistency and natural extension that were developed for sets of desirable option sets by De Bock and De Cooman \cite{de2018desirability}.
To that end, for any assessment $\A\in\QQ$, we consider the set $\Ka(\A)\coloneqq\{K\in\Ka\colon\A\subseteq K\}$ of all coherent sets of desirable option sets that contain $\A$ and let
\begin{equation}\label{eq:defEx}
\Ex(\A)\coloneqq\bigcap\Ka(\A),
\end{equation}
where we use the convention that $\bigcap \emptyset=\Q$.

\begin{definition}\label{def:natext}
For any assessment $\A\in \QQ$, we say that $\A$ is \emph{consistent} if $\Ka(\A)\neq\emptyset$ and we then call $\Ex(\A)$ the \textit{natural extension} of $\A$ {\cite[definition~9]{de2018desirability}}.
\end{definition}
As proven by De Bock and De Cooman \cite{de2018desirability}, the consistency of $\A$ implies the coherence of $\Ex(\A)$.
Our next result establishes that the converse is true as well.
\begin{theorem}\label{th:consifandonlyifcoh}
For any assessment $\A\in\QQ$ the following are equivalent:
\begin{enumerate}
\item $\A$ is consistent;
\item $\Ex(\A)$ is coherent;
\item $\emptyset\notin\Ex(\A)$.
\end{enumerate}
\end{theorem}
The connection with choice functions is established by the following result.

\begin{theorem}\label{th:natequiv}
Let $C$ be a choice function. Then $C$ is consistent if and only if $\emptyset \notin \Ex(\A_C)$, and if it is, then
\(
C_{\Ex(\A_{C})}(A)=\Ex(C)(A) \text{ for all $A\in\Q$}.
\) 
\end{theorem}
By this theorem, we see that checking consistency and computing the natural extension of a choice function $C$ amounts to being able to check for any option set whether it belongs to $\Ex(\A_C)$ or not.
Indeed, given some choice function $C$, we can check its consistency by checking if $\emptyset\in\Ex(\A_C)$ and, taking into account \cref{eq:CK}, we can calculate $\Ex(C)(A)$ for a given set $A\in\Q$, by checking for any element $u\in A$ if $(A - u)\setminus\{0\} \in \Ex(\A_{C})$.

\section{Natural extension and consistency for finite assessments}
\label{sec:choose}
In practice, we will typically be interested in computing the natural extension $\Ex(C)$ of choice functions $C$ that are finitely generated, in the sense that $C(A)\neq A$ for only finitely many option sets $A\in\Q$.
In that case, if we let $\mathcal{O}_C\coloneqq\{A\in\Q\colon C(A)\neq A\}$, then
\[
\A_C=\{C(A)-u\colon A\in\Q,u\in R_C(A)\}=\{C(A)-u\colon A\in\mathcal{O}_C,u\in R_C(A)\}.
\]
Since $\mathcal{O}_C$ is finite and every $R_C(A)\subseteq A$ is finite because $A$ is, it then follows that $\A_C$ is finite as well.
Without loss of generality, for finitely generated choice functions $C$, $\A_C$ will therefore be of the form $\A_C=\{\AAA\}$, with $n\in\N\cup\{0\}$.
The list of option sets $\AAA$ may contain duplicates in practice, for example if different $A\in \mathcal{O}_C$ and $u\in A$ yield the same option set $C(A)-u$.
It is better to remove these duplicates, but our results will not require this.
The only thing that we will assume is $\A_C=\{\AAA\}$.

The following theorem is the main result that will allow us to check in practice whether an option set belongs to $\Ex(\A_C)$ or not.

\begin{theorem}\label{th:testnatext}
Let $\A=\{\AAA\}$, with $\AAA\in\Q$ and $n\in\N\cup\{0\}$.
An option set $S\in\Q$ then belongs to $\Ex(\A)$ if and only if either $S \cap \V_{\succc0}\neq \emptyset$ or, $n\neq 0$ and, for every $\uu \in \times_{j=1}^n A_j$, there is some $s\in S\cup\{0\}$ and $\lala\in \R^{n,+}$ for which $\lala \uu \leq s$.
\end{theorem}
In combination with \cref{th:natequiv}, this result enables us to check the consistency and compute the natural extension of finitely generated choice functions.
Checking consistency is equivalent to checking if $\emptyset\notin\Ex(\A_C)$.

\begin{myexpcont}
We will now go ahead and test if the strategic advisor was at least consistent in his choices.
Since $\emptyset\cap \V_{>0}=\emptyset$, we have to check the second condition in \cref{th:testnatext}.
In particular, we have to check if, for every tuple $\uu \in A_1\times A_2 \times A_3$, there is some $\lala\in\R^{3,+}$ such that $\lala \uu\leq 0$.
We will show that this is not the case for the particular tuple $\uu=(w_2,w_4,w_5)\in A_1\times A_2 \times A_3$. Assume that there is some $\lala=(\lambda_1,\lambda_2,\lambda_3)\in \R^{3,+}$ such that
\(
\lambda_1 w_2+\lambda_2 w_4 + \lambda_3 w_5=\lala \uu\leq 0.
\)
Notice that $2w_4\leq 5w_2$, so if we let $\mu_1\coloneqq\frac25\lambda_1+\lambda_2$ and $\mu_2\coloneqq7\lambda_3$ then  
\[
\lala \uu\geqslant\frac25\lambda_1 w_4+\lambda_2 w_4 +\lambda_3  w_5 =\mu_1 w_2 +\frac17 \mu_2  w_5=
(5 \mu_1 - \mu_2,
-3 \mu_1 + \mu_2)
.
\]
Since $\lala \uu \leq 0$, this implies that $5\mu_1\leq\mu_2\leq 3\mu_1$ and thus $\mu_1\leq0$ and $\mu_2\leq0$.
This is impossible though because $\lala\in\R^{3,+}$ implies that $\mu_1>0$ or $\mu_2>0$.
Hence, $\emptyset\notin\Ex(\A_C)$, so $\A_C$ is consistent.
We conclude that the decisions of the strategic advisor were consistent, enabling us to use natural extension to study their consequences.
\end{myexpcont}

If a finitely generated choice function $C$ is consistent, then for any option set $\A\in\Q$, we can evaluate $\Ex(C)(A)$ by checking for every individual $u\in A$ if $u\in\Ex(C)(A)$. 
As we know from \cref{th:natequiv,eq:CK}, this will be the case if and only if $(A-u)\setminus\{0\}\notin\Ex(\A_C)$.
\begin{myexpcont}
We can now finally tackle the problem at the end of Example~1.3: choosing from the set $B_3=\{(-3,4),\allowbreak (0,1),\allowbreak (4,-3)\}$.
This comes down to computing $\Ex(C)(B_3)$.
Because of \cref{th:natequiv}, we can check if $(4,-3)$ is rejected from $B_3$ by checking if $\{(-7,7),(-4,4)\}\in \Ex(\A_C)$.
By \cref{th:testnatext}, this requires us to check if for every $\uu\in A_1\times A_2 \times A_3$ we can find some $s\in \{(-7,7),(-4,4),(0,0)\}$ and some $\lala\in \R^{3,+}$ such that $\lala \uu \leq s$.
Since $(-7,7)=w_5\in\{w_5\}=A_3$, $s=(-7,7)$ and $\lala=(0,0,1)$ do the job for every $\uu$.
So we can reject $(4,-3)$.
Checking if $(0,1)$ is rejected is analogous: we have to check if $\{(-3,3),(4,-4)\}\in \Ex(\A)$.
In this case, we can use $s=(-3,3)$ and $\lala=(0,0,\frac37)$ for every $\uu$ in $A_1\times A_2 \times A_3$ to conclude that $(0,1)$ is rejected as well.
Since $\Ex(C)(B_3)$ must contain at least one option because of $C_0$ (which applies because \cref{th:coh} and the consistency of $C$ imply that $\Ex(C)$ is coherent) it follows that $\Ex(C)(B_3)=\{(-3,4)\}$.
So based on the advisor's earlier decisions and the axioms of coherence, it follows that the company should choose $(-3,4)$ from $B_3$.
\end{myexpcont}
In this simple toy example, the assessment $\A_C$ was small and the conditions in \cref{th:testnatext} could be checked manually.
In realistic problems, this may not be the case though.
To address this, we will provide in \cref{sec:prac} an algorithm for checking the conditions in \cref{th:testnatext}.
But first, we provide methods for reducing the size of an assessment.

\section{Simplifying assessments}
\label{sec:Ass}
For any given assessment $\A=\{\AAA\}$, the number of conditions that we have to check to apply \cref{th:testnatext} is proportional to the size of $\times_{j=1}^n A_j$.
Since \cref{th:testnatext} draws conclusions about $\Ex(\A)$ rather than $\AAA$, it can thus be useful to try to make $\times_{j=1}^n A_j$ smaller without altering $\Ex(\A)$, as this will reduce the number of conditions that we have to check.
This is especially true when we want to apply \cref{th:testnatext} to several sets $S$, for example because we want to evaluate the natural extension of a choice function in multiple option sets.

To make $\times_{j=1}^n A_j$ smaller, a first straightforward step is to remove duplicates from $\AAA$; after all, it is only the set $\A=\{\AAA\}$ that matters, not the list $\AAA$ that generates it.
To further reduce the size of $\times_{j=1}^n A_j$, we need to reduce the size of $\A$ and the option sets it consists of.
To that end, we introduce a notion of equivalence for assessments.
\begin{definition}\label{def:eqAss}
For any two assessments $\A_1,\A_2\in\QQ$, we say that $\A_1$ and $\A_2$ are \emph{equivalent} if $\Ka(\A_1)=\Ka(\A_2)$.
\end{definition}
It follows immediately from \cref{def:natext,eq:defEx} that replacing an assessment with an equivalent one does not alter its consistency, nor, if it is consistent, its natural extension.

The following result shows that it is not necessary to directly simplify a complete assessment; it suffices to focus on simplifying subsets of assessments. 
\begin{proposition}\label{prop:simpindiv}
If two assessments $\A_1,\A_2\in\QQ$ are equivalent and $\A_1 \subseteq \A\in\QQ$, then $\A$ is equivalent to $(\A\setminus \A_1) \cup \A_2$.
\end{proposition}
This result is important in practice because it means that we can build and simplify assessments gradually when new information arrives and that we can develop and use equivalence results that apply to small (subsets of) assessments.

A first simple such equivalence result is that we can always remove zero from any option set.
\begin{proposition}\label{prop:wegMetDieNul}
Consider an option set $A\in\Q$.
Then the assessment $\{A\}$ is equivalent to $\{A\setminus\{0\}\}$.
\end{proposition}
This result can be generalized so as to remove options for which there is a second option that can, by scaling, be made better than the first option.
\begin{theorem}\label{th:simp1}
Consider an option set $A\in\Q$ and two options $u,v\in A$ such that $u\neq v$ and $u\leq \mu v$ for some $\mu\geq 0$.
Then the assessment $\{A\}$ is equivalent to $\{A\setminus \{u\}\}$. 
\end{theorem}
If $\V=\R^2$, this result---together with \cref{prop:simpindiv}---guarantees that every option set in $\A$ can be replaced by an equivalent option set of size at most two.
\begin{proposition}\label{prop:simp1inR2}
Let $\V=\R^2$ and consider any option set $A\in\Q$. 
Then there is always an option set $B\in \Q$ with at most two options such that $\{A\}$ is equivalent to $\{B\}$, and this option set $B$ can be found by repeated application of \cref{th:simp1}.
\end{proposition}
In the case of our running example, all the option sets in $\A$ can even be reduced to singletons.
\begin{myexpcont}
In $A_1$ we see that $w_1=2 w_2$ and in $A_2$ we see that $w_4 \leq \frac57 w_3$.
So, by \cref{prop:simpindiv,th:simp1}, we can simplify the assessment $\A_C$ of Example 1.4 to
$\A'_C\coloneqq\{\{w_2\},\{w_3\},\{w_5\}\}$.
\end{myexpcont}

Our equivalence results so far were all concerned with removing options from the option sets in $\A$.
Our next result goes even further: it provides conditions for removing the option sets themselves.
\begin{theorem}\label{th:simp2}
Consider an assessments $\A$ and an option set $A\in\A$ such that $A\in\Ex(\A\setminus\{A\})$. Then $\A$ is equivalent to $\A\setminus\{A\}$.
\end{theorem}

\begin{myexpcont}
Let us start from the assessment $\A'_C=\{\{w_2\},\allowbreak\{w_3\},\allowbreak\{w_5\}\}$ in Example~1.7. 
We will remove $\{w_2\}$ from this assessment using \cref{th:simp2}.
To do that, we need to show that $\{w_2\}\in \Ex(\{\{w_3\},\{w_5\}\})$.
To that end, we apply \cref{th:testnatext} for $S=\{w_2\}$.
Since $w_2\not>0$, it follows that $S\cap \V_{>0}=\emptyset$.
We therefore check the second condition of \cref{th:testnatext}.
Since $\{w_3\}\times\{w_5\}$ is a singleton, we only need to check the condition for a single tuple $\uu=(w_3,w_5)$.
For $s=w_2$ and $\lala=(\frac14,0)$, we find that $\lala \uu=\frac14 w_3+0 w_5=(\frac74,-1)\leq(2,-1)= w_2=s$.
Hence, $\{w_2\}\in \Ex(\{\{w_3\},\{w_5\}\})$ and we can therefore replace $\A'_C$ by the smaller yet equivalent assessment
$
\A''_C=\{\{w_3\},\{w_5\}\}.
$
Taking into account our findings in Example 1.7, it follows that $\A_C$ is equivalent to $\A_C''$ and, therefore, that $\Ex(\A_C)=\Ex(\A_C'')$.

Obviously, this makes it easier to evaluate $\Ex(C)$.
Suppose for example that we are asked to choose from $\{v_7,v_8\}$, with $v_7\coloneqq(5,-2)$ and $v_8\coloneqq(-4,3)$.
By \cref{th:natequiv,eq:CK}, we can check if we can reject $v_8$ by checking if $\{v_7-v_8\}=\{(9,-5)\}\in\Ex(\A''_C)=\Ex(\A_C)$.
For $s=(9,-5)$, $\uu=(w_3,w_5)$ and $\lala=(\frac97,0)$ we see that $\lala \uu=(9,-\frac{36}{7}) \leq s$, so \cref{th:testnatext} tells us that, indeed, $\{(9,-5)\}\in\Ex(\A''_C)$.
If we were to perform the same check directly for $\Ex(\A_C)$, then we would have to establish four inequalities---one for every $\uu \in A_1\times A_2 \times A_3$---while now we only had to establish one.
\end{myexpcont}

\section{Practical implementation}
\label{sec:prac}

For large or complicated assessments, it will no longer be feasible to manually check the conditions in \cref{th:testnatext}.
In those cases, the algorithm that we are about to present can be used instead. 
According to \cref{th:testnatext}, testing if an option set $S$ belongs to $\Ex(\A)$ for some assessment sequence $\A=\{\AAA\}$ requires us to check if $S\cap \V_{>0}\neq \emptyset$ and, in case $n\neq0$, if there is for every $\uu\in\times_{j=1}^n A_j$ some $s\in S\cup\{0\}$ and $\lala\in\R^{n,+}$ such that $\lala \uu \leq s$.
If one of these two conditions is satisfied, then $S$ belongs to $\Ex(\A)$.
The first condition is not complicated, as we just have to check for every $s\in S$ if $s>0$.
For the second condition, the difficult part is how to verify, for any given $\uu\in \times_{j=1}^n A_j$ and $s\in S\cup\{0\}$, whether there is some $\lala\in \R^{n,+}$ for which $\lala \uu \leq s$.
If one of these two conditions is satisfied then $S$ belongs to $\Ex(\A)$.
Given the importance of this basic step, we introduce a boolean function \textsc{IsFeasible}$\colon \V^n\times \V\to\{\text{True},\text{False}\}$. For every $\uu\in\V^n$ and $s\in\V$, it returns True if $\lala \uu \leq s$ for at least one $\lala\in \R^{n,+}$, and False otherwise.

So the only problem left is how to compute \textsc{IsFeasible}$(\uu,s)$.
The tricky part is the constraint that $\lala=(\lambda_1,\dots,\lambda_n)\in \R^{n,+}$. 
By definition of $\R^{n,+}$, this can be rewritten as $\lambda_j\geq0$ for all $j\in\{1,\dots,n\}$ and $\sum_{j=1}^n \lambda_j>0$, which are all linear constraints.
Since the condition $\lala \uu \leq s$ is linear as well, we have a linear feasibility problem to solve.
However, strict inequalities such as $\sum_{j=1}^n \lambda_j>0$ are problematic for software implementations of linear feasibility problems.
A quick fix is to choose some very small $\epsilon>0$ and impose the inequality $\sum_{j=1}^n \lambda_j\geq \epsilon$ instead, but since this is an approximation, it does not guarantee that the result is correct.
A better solution is to use the following alternative characterisation that, by introducing an extra free variable, avoids the need for strict inequalities.\footnote{This result was inspired by a similar trick that Erik Quaegebeur employed in his CONEstrip algorithm \cite{quaeghebeur2013conestrip}.}
\begin{proposition}\label{prop:hunt}
Consider any $s\in \V$ and any $\uu=(u_1,\dots,u_n)\in\V^n$.
Then \textsc{IsFeasible}$(\uu,s)=$\emph{True} if and only if there is a $\mumu=(\mu_1,\dots,\mu_{n+1})\in\R^{n+1}$ such that $\sum_{j=1}^n \mu_j u_j \leq \mu_{n+1} s$, $\mu_j\geq 0$ for all $j\in\{1,\dots,n\}$, $\mu_{n+1}\geq 1$ and $\sum_{j=1}^{n} \mu_j\geq 1$.
\end{proposition}
Computing \textsc{IsFeasible}$(\uu,s)$ is therefore a matter of solving the following linear feasibility problem: 
\begin{align*}
  \text{find}\quad        & \mu_1,\dots,\mu_{n+1}\in\R, \\
  \text{subject to\quad}  & \mu_{n+1} s(x)-\textstyle\sum_{j=1}^n \mu_j u_j(x)\geq 0  &&\text{for all } x\in\X,\\
                    & \textstyle\sum_{j=1}^n \mu_j\geq 1,\quad \mu_{n+1}\geq1, \\
                \text{and}\quad  & \mu_j\geq0    &&\text{for all } j\in\{1,\dots,n\}.
\end{align*}
For finite $\X$, such problems can be solved by standard linear programming methods; see for example \cite{Dantzig}.

\begin{algorithm}
\caption{Check if an option set $S\in\Q$ is in $\Ex(\A)$ for an assessment $\A\in\QQ$\label{alg:IsInNat}}
\begin{algorithmic}[1]
\Require{$S,\AAA\in\Q$ with $n\in \N\cup\{0\}$ and $\A=\{\AAA\}$.}
\Statex
\Function{IsInExtension}{$S,\AAA$}\Comment{Check if $S$ is in $\Ex(\A)$.}
    \ForAll{$s \in S$}
        \If{$s>0$ (so $s\geqslant0$ and $s\neq 0$)}
          \State\Return{True}
        \EndIf
    \EndFor
    \If{$n=0$}
        \State\Return{False}
    \EndIf
    \ForAll{$\uu \in \times_{j=1}^n A_j$}
        \ForAll{$s\in S\cup\{0\}$}\Comment{Search an $s$ and $\lala$ such that $\lala \uu\leq s$.}
            \Let{res}{\textsc{IsFeasible}$(\uu,s)$}
            \If{res}
                \State\textbf{break} \Comment{Stop the loop in $s$ when a suitable $s$ is found.}
            \EndIf
        \EndFor
        \If{$\neg$ res}
            \State\Return{False} \Comment{If there is no such $s$.}
        \EndIf
    \EndFor
    \State\Return{True} \Comment{When all $\uu$ have been checked.}
\EndFunction
\end{algorithmic}
\end{algorithm}
Together, \cref{th:testnatext,prop:hunt} therefore provide a practical method to test if an option set $S$ belongs to $\Ex(\A)$, with $\A=\{\AAA\}$.
Pseudocode for this method is given in \cref{alg:IsInNat}.
First, in lines 2 to 4, we check if $S\cap \V_{>0}\neq \emptyset$.
If it is, then $S\in\Ex(\A)$ and we thus return True.
If this is not the case, and the assessment is empty, i.e. $n=0$, then we have to return False, as we do in lines 5 and 6.
Next, in lines 7 to 11, we run through all $\uu\in\times_{j=1}^n A_j$ and search an $s\in S\cup\{0\}$ for which \textsc{IsFeasible}$(u,s)=$ True, i.e. for which there is some $\lala\in\R^{n,+}$ such that $\lala \uu\leq s$.
As soon as we have found such an $s$, we can break from the loop---halt the search---and go to the next $\uu$.
However, if we went through all of $S\cup\{0\}$ and we did not find such an $s$, then the second condition of \cref{th:testnatext} is false for the current $\uu$ and thus $S$ does not belong to $\Ex(\A)$; we then return False, as in lines 12 and 13.
On the other hand, if we went through all $\uu\in\times_{j=1}^n A_j$ and for every one of them we have found an $s\in S\cup\{0\}$ such that \textsc{IsFeasible}$(\uu,s)$=True, then we conclude that $S$ is in $\Ex(\A)$ by the second condition of \cref{th:testnatext}, so we return True.


\addtocounter{example}{1}
\begin{example}
Roger is an expert in the pro snooker scene.
An important game is coming soon where two players will play two matches.
Betting sites will offer bets on the following three possible outcomes: the first player wins 2-0, a 1-1 draw, or the second player wins 2-0.
So a bet corresponds with an option in $\R^3$, the components of which are the rewards for each of the three outcomes.
Before the possible bets are put online, we ask Roger to provide us with an assessment.
He agrees to do so, but tells us that we should not contact him again when the bets are online.
\sloppy
We ask him to choose from the sets
\(
B_1=\{(-4, 1, -1), \allowbreak(3, -5, -1), \allowbreak (-3, 1, -1), \allowbreak(4, 0, -4),  \allowbreak (3, -5, 4)\}
\),
\(
B_2=\{(-4, 2, 4),\allowbreak (-2, -4, 3),\allowbreak (0, -4, 2),\allowbreak (0, 3, -5),\allowbreak (2, 1, 3)\}
\)
and
\(
B_3=\{(-4, 1, 4),\allowbreak (-2, -2, 4),\allowbreak (-5, 3, 4)\}
\).
He provides us with the choice assessment
\(
C(B_1)=\{(4, 0, -4),\allowbreak (3, -5, 4)\},
\)
\(
C(B_2)=\{(-4, 2, 4)\}
\) and 
\(
C(B_3)=\{(-4, 1, 4),\allowbreak (-2, -2, 4)\}.
\)
Some time later, the betting site makes the following set of bets available:
\begin{multline*}
A=\{(-1, -1, 2),\allowbreak (-4, -4, 6),\allowbreak (-2, -10, 6),\allowbreak (-1, 0, -2),\allowbreak (-2, 8, -6),\\\allowbreak (2, -4, 4),\allowbreak (4, -6, 1),\allowbreak (-3, 8, 5),\allowbreak (2, 9, -9),\allowbreak(1, 7, -3)\}.
\end{multline*}
The question then is what bet to choose from $A$, based on what Roger has told us. The assessment $\A_C$ that correponds to Roger's choice statements contains eight option sets. However, we can use~\cref{th:simp1,th:simp2,prop:simpindiv} to reduce $\A_C$ to the equivalent assessment
\(
\A^*=
\{
\{(3, -5, 0)\},\allowbreak
\{(-6, 1, 1), (-2, 2, -8)\}
\}.
\)
With \cref{th:consifandonlyifcoh,th:natequiv,alg:IsInNat}, we find that the assessment is consistent and with subsequent runs of \cref{alg:IsInNat}, and using \cref{th:natequiv,eq:CK}, we find that
\(
\Ex(C)(A)=\{(-4, -4, 6),\allowbreak (-2, -10, 6),\allowbreak (-3, 8, 5)\}.
\)
So we can greatly reduce the number of bets to choose from but we cannot, based on the available information, deduce entirely what Roger would have chosen.
\hfill$\lozenge$
\end{example}

\section{Conclusion and future work}\label{sec:conc}

The main conclusion of this work is that choice functions provide a principled as well as practically feasible framework for inferring new decisions from previous ones.
The two key concepts that we introduced to achieve this were consistency and natural extension.
The former allows one to check if an assessment of choices is compatible with the axioms of coherence, while the latter allows one to use these axioms to infer new choices.
From a practical point of view, our main contribution is an algorithm that is able to execute both tasks.
The key technical result that led to this algorithm consisted in establishing a connection with the framework of sets of desirable option sets.
This allowed us to transform an assessment of choices into an assessment of desirable option sets, then simplify this assessment, and finally execute the desired tasks directly in this setting.


Future work could add onto \cref{sec:Ass} by trying to obtain a `simplest' representation for any given assessment, thereby further reducing the computational complexity of our algorithm.
We would also like to conduct extensive experiments to empirically evaluate the time efficiency of our algorithm and the proposed simplifications, and how this efficiency scales with a number of important parameters. This includes the number of option sets in an assessment, the size of the option sets themselves, the dimension of the vector space $\V$ and the size of the option set $A$ for which we want to evaluate the natural extension $\Ex(C)(A)$.
We also intend to consider  alternative forms of assessments, such as bounds on probabilities, bounds on expectations and preference statements, and show how they can be made to fit in our framework.
Finally, we would like to apply our methods to a real-life decision problem.
\subsubsection*{Acknowledgements}
The work of Jasper De Bock was supported by his BOF Starting Grant “Rational decision making under uncertainty: a new paradigm based on choice functions”, number 01N04819. We also thank the reviewers for their valuable feedback.
%
%
%
%
\bibliographystyle{splncs04}
\bibliography{ArneDecadt}

\iftoggle{extendedversion}{}{

\section*{Appendix}
\addcontentsline{toc}{section}{Appendix}

\begin{proofof}{\cref{th:coh}}
We first consider the case where $C$ is consistent.
The fact that $C'(A)\subseteq\Ex(C)(A)$ for every $A\in\Q$ and $C'\in\mathcal{C}_C$ follows directly from the definition of $\Ex(C)$.
So the only thing left to prove in this case is that $\Ex(C)\in \mathcal{C}_C$, i.e., that $\Ex(C)$ is coherent.

Since $C$ is consistent, $\mathcal{C}_C$ is non-empty.
Hence, $\Ex(C)$ is a non-empty union of coherent choice functions.
We will prove that all coherence axioms are preserved under such non-empty unions.
Axiom $\mathrm{C}_0$ and $\mathrm{C}_1$ are trivially preserved under such non-empty unions.
For the others, observe that by De Morgan's laws we have for all $A\in \Q$ that
\begin{equation}\label{eq:Rej}
R_{\Ex(C)}(A)=\bigcap_{C'\in\mathcal{C}_C} R_{C'}(A).
\end{equation}
This already implies that $\mathrm{C}_4$ holds.
Indeed, for any $A,B\in\Q$ such that $A\subseteq B$, we know for all $C'\in\mathcal{C}_C$ that $R_{C'}(A)\subseteq R_{C'}(B)$---because $C'$ satisfies $\mathrm{C}_4$---and it therefore follows from \cref{eq:Rej} that also $R_{\Ex(C)}(A)\subseteq R_{\Ex(C)(B)}$.
\cref{eq:Rej} also implies that
\begin{align}\label{eq:Kintersect}
K_{\Ex(C)}
&=\{A\in\Q\colon 0\in R_{\Ex(C)}(A\cup\{0\})\}\notag\\
&=\left\{A\in \Q \colon  0\in  \bigcap_{C'\in\mathcal{C}_C} R_{C'}(A\cup\{0\})\right\}\notag\\
&=\bigcap_{C'\in\mathcal{C}_C}\left\{A\in \Q \colon  0\in R_{C'}(A\cup\{0\})\right\}=\bigcap_{C'\in\mathcal{C}_C} K_{C'}.
\end{align}
To prove $\mathrm{C}_2$, we consider any $u\in \V_{>0}$ and prove that $\{u\}\in K_{\Ex(C)}$.
By $\mathrm{C}_2$, for every coherent choice function $C' \in \mathcal{C}_C$, we know that $\{u\}\in K_{C'}$. 
Hence, indeed, $\{u\}\in \bigcap_{C'\in\mathcal{C}_C} K_{C'}=K_{\Ex(C)}$, using equation \cref{eq:Kintersect} for the last equality.

To prove $\mathrm{C}_3$, we consider any $A,B\in K_{\Ex(C)}$ and $\lala\colon A\times B\to \R^{2,+}$ and prove that $\{\lala(\uu)\uu\colon \uu\in A\times B\}\in K_{\Ex(C)}.$
For any $C'\in \mathcal{C}_{C}$, since $A,B\in K_{\Ex(C)}$, we know from \cref{eq:Kintersect} that also $A,B\in K_{C'}$.
Since $C'$ is coherent, it therefore follows from $\mathrm{C}_3$ that $\{\lala(\uu)\uu\colon\uu\in A\times B\}\in K_{C'}$.
Since this is true for every $C'\in\mathcal{C}_C$, \cref{eq:Kintersect} implies that also $\{\lala(\uu)\uu\colon \uu\in A\times B\}\in K_{\Ex(C)}$, as required.


Finally, we consider the case where $C$ is not consistent.
Then by definition of consistency $\mathcal{C}_C=\emptyset$, and it therefore follows from the definition of $\Ex(C)$ that $\Ex(C)(A)=\emptyset$ for all $A\in\Q$.
Since this is incompatible with $C_0$, we conclude that $\Ex(C)$ is incoherent.
\end{proofof}

\begin{proposition}\label{prop:HeTBeSTaAtNoGnIeTPeRsE}
Consider any set of desirable option sets $K\subseteq \Q$ and any choice function $C$ that are connected by 
\begin{equation}\label{eq:connect}
(\forall u\in\V)(\forall A\in\Q)\left(u\notin C(A\cup\{u\})\Leftrightarrow A-u\in K\right).
\end{equation}
Then $K$ is coherent if and only if $C$ is \cite[Proposition~4]{de2019interpreting}.
\end{proposition}

\begin{proofof}{\cref{prop:KcohCK}}
We will prove that $K$ and $C_K$ are connected by \cref{eq:connect}.
The result then follows from \cref{prop:HeTBeSTaAtNoGnIeTPeRsE}.

For the implication to the left, take any $A\in\Q$ and $u\in \V$ and assume that $A-u\in K$.
Then $A\cup\{u\}-u=(A-u)\cup\{0\} \in K$ by axiom $\mathrm{K}_4$ and therefore also $((A\cup\{u\})-u)\setminus\{0\}\in K$ by axiom $\mathrm{K}_0$.
So by definition of $C_K$, $u\notin C_K(A\cup\{u\})$.

For the implication to the right, consider any $A\in \Q$ an $u\in \V$ and assume that $u\notin C_K(A\cup\{u\})$.
Then $((A\cup\{u\})-u)\setminus\{0\}\in K$ by definition of $C_K$.
Since $((A\cup\{u\})-u)\setminus\{0\}=((A-u)\cup\{0\})\setminus\{0\}=(A-u)\setminus\{0\}$, this implies that $(A-u)\setminus\{0\}\in K$ and thus by axiom $\mathrm{K}_4$ also $A-u\in K$.
\end{proofof}

\begin{proofof}{\cref{prop:CcohKC}}
We will prove that $C$ and $K_C$ are connected by \cref{eq:connect}.
The result then follows from \cref{prop:HeTBeSTaAtNoGnIeTPeRsE}.
Take any $u\in\V$ and $A\in\Q$. 
Then the following equivalences hold:
\begin{equation*}
\hspace{1.8cm}
\begin{aligned}
&& u\notin C(A\cup\{u\})&\Leftrightarrow 0 \notin C((A\cup\{u\})-u)&\text{by axiom $\mathrm{C}_1$,}\\
 && &\Leftrightarrow 0 \notin C((A-u)\cup\{0\})&\text{by definition of $A-u$,}\\
 && &\Leftrightarrow 0 \in K_C &\text{by definition of $K_C$.}
\end{aligned}
\end{equation*}
\end{proofof}

\begin{proofof}{\cref{th:unique}}
We first prove that $\Phi$ is injective.
$\Phi$ maps to $\Ka$ by \cref{prop:CcohKC}.
Assume \emph{ex absurdo} that there are two different coherent choice function $C$ and $C'$ such that $K_{C}=K_{C'}$.
Then take one set $A\in\Q$ for which $C(A)\neq C'(A)$ and, without loss of generality, assume that $C(A)$ has an element $u$ that is not in $C'(A)$.
Then by $\mathrm{C}_1$, it holds that $0\in C(A-u)$ but $0\notin C'(A-u)$.
Moreover, $A-u$ will contain $0$ since $u\in A$, so $A-u=(A-u)\cup\{0\}$ and it follows that $A-u\notin K_C$ but $A-u\in K_{C'}$, contradicting the fact that $K_C=K_{C'}$.

Next, we prove that $\Phi$ is surjective.
Take any $K\in \Ka$.
We need to prove that there is some $C\in \mathcal{C}$ such that $\Phi(C)=K$.
By \cref{prop:KcohCK}, $C_K$ is coherent, so if we let $C\coloneqq C_K$, then $C\in\mathcal{C}$.
Furthermore, $\Phi(C)=K$ because
\begin{multline*}
\Phi(C_K)=K_{C_K}=\{A\in\Q \colon 0\notin C_K(A\cup\{0\})\}\\=\{A\in\Q\colon((A\cup\{0\})-0)\setminus\{0\}\in K\}=\{A\in\Q \colon A\setminus\{0\}\in K\}=K,
\end{multline*}
using axiom $\mathrm{K}_0$ for the last equality.

Since $\Phi$ is a bijection and, as we proved above, $\Phi(C_K)=K$ for any $K\in\Ka$, we also know that $\Phi^{-1}(K)=C_K$ for any $K\in\Ka$.
\end{proofof}

\begin{proofof}{\cref{th:consifandonlyifcoh}}
We give a circular proof. 
The implication that if $\A$ is consistent then $\Ex(\A)$ is coherent follows straight from \cite[theorem~8]{de2018desirability} and the definitions of consistency and $\Ex(\A)$.
If $\Ex(\A)$ is coherent, then $\{0\}\notin \Ex(\A)$ by $\mathrm{K}_1$, and therefore $\emptyset\notin \Ex(\A)$ by $\mathrm{K}_4$.
Finally, if $\emptyset\notin \Ex(\A)$, then $\Ex(\A)\neq\Q$ and therefore, because of \cref{eq:defEx}, $\Ka(\A)\neq \emptyset$; hence $\A$ is consistent.
\end{proofof}

Some of our results and proofs will make use of expressions that are similar to the ones in axiom $\mathrm{K}_3$, but for positive linear combinations of more than two sets.
To improve the readability of these results and proofs we introduce an operator $\Posi'$, defined for any finite sequence $\AAA$, with $n$ a positive integer, by
\[
\Posi'(\AAA)\coloneqq\left\{ \{\boldsymbol{\lambda}(\mathbf{u}) \mathbf{u}\colon \mathbf{u}\in \times_{j=1}^n A_j\} \colon \boldsymbol{\lambda}\in\Lamp_n(\times_{j=1}^n A_j) \right\},
\]
where $\Lamp_n(\times_{j=1}^n A_j)$ is the set of all functions from $\times_{j=1}^n A_j$ to $\R^{n,+}$.
This $\Posi'$ operator is closely related to the $\Posi$ operator in Reference \cite{de2018desirability}, defined for all $\A\in\QQ$ by
\[
\Posi(\A) \coloneqq \left\{\{\boldsymbol{\lambda}(\uu)\uu\colon \uu\in \times_{k=1}^n A_k\}\colon n\in\N; A_1,\ldots,A_n\in \A; \boldsymbol{\lambda}\in \Lamp_n(\times_{k=1}^n A_k)\right\}.
\]
The same reference also considers the operator $\Rs$, defined for all $\A\in\QQ$ by
\[
\Rs(\A)\coloneqq \{A\in\Q\colon (\exists B\in \A)\; B\setminus \V_{\leq0}\subseteq A\},
\]
where $\V_{\leq0}\coloneqq\{u\in \V \colon u\leq 0\}$.
Together, these two operators can be used to provide the following alternative characterisation for the consistency and natural extension of a set of desirable option sets.

\begin{theorem}\label{th:natext}
Consider any assessment $\A\in\QQ$. Then $\A$ is consistent if and only if $\emptyset \notin \A$ and $\{0\}\notin\Posi(\V^s_{\succc0}\cup \A)$.
Moreover, if $\A$ is consistent, then
\(
\Ex(\A)=\Rs(\Posi(\V^s_{\succc0}\cup \A))
\)
\cite[theorem~10]{de2018desirability}.
\end{theorem}
\begin{corollary}\label{cor:natext}
For any assessment $\A\in\QQ$, $\Ex(\A)=\Rs(\Posi(\V^s_{>0}\cup \A))$.
\end{corollary}
\begin{proof}
If $\A$ is consistent, this result follows from the second part of \cref{th:natext}.
If $\A$ is not consistent, then we know from the first part of \cref{th:natext} that $\emptyset\in \A$ or $\{0\}\in\Posi(\V_{>0}^s\cup\A)$.
If $\emptyset\in \A$, then $\emptyset\in \Posi(\V_{>0}^s\cup\A)$ and therefore $\Rs(\Posi(\V_{>0}^s\cup\A))=\Q$.
If $\{0\}\in \Posi(\V_{>0}^s\cup\A)$, then also $\Rs(\Posi(\V_{>0}^s\cup\A))=\Q$.
Hence, in both cases $\Rs(\Posi(\V_{>0}^s\cup\A))=\Q=\Ex(\A)$, where the second equality follows from \cref{eq:defEx} and the inconsistency of $\A$.
\end{proof}




\begin{theorem}\label{th:speedyAss}
For any two disjoint option sets $A,B\in \Q$ and any coherent set of desirable option sets $K$ we have that
\[
(\forall u\in A)\; (A\cup B)-u \in K \Leftrightarrow (\forall u\in A)\;B-u \in K.
\]
\end{theorem}

\begin{proof}
The implication to the left follows immediately from axiom $\mathrm{K}_4$.
So it remains to prove the implication to the right.
Without loss of generality, let us write $A=\{v_1,\dots,v_m\}$ and $B=\{v_{m+1},\dots,v_n\}$, so that we can index their elements.
If $m=0$, then the result is trivial, so we assume that $m>0$ for the rest of this proof.
For every $k\in\{1,\dots,m\}$ we let $A_k\coloneqq(A\cup B)-v_k=\{v_1-v_k,\dots,v_n-v_k\}$.
Assuming that $A_k\in K$ for all $k\in\{1,\dots,m\}$, we now need to prove that $B-v_k\in K$ for all $k\in\{1,\dots,m\}$. 
So fix any $k^* \in \{1,\dots,m\}$.
We will prove that $B-v_{k^*}\in K$.
With any element $\uu$ of $\times_{k=1}^m A_k$ we can associate a map $\ell\colon\{1,\dots,n\}\to\{1,\dots,n\}$ such that $\uu=(v_{\ell(1)}-v_1,\ldots,v_{\ell(m)}-v_m)$ and $\ell(k)=k$ for $k>m$.
Hence, for all $k\in\{1,\dots,m\}$, we have that $u_k=v_{\ell(k)}-v_k$.
Furthermore, as a result of our definition of $\ell$, we can also define $u_{k}\coloneqq v_{\ell(k)}-v_{k}=0$ for $k\in\{m+1,\dots,n\}$.
For any $j\in \mathbb{N}\cup\{0\}$, we will also the consider the $j$-th iterate of $\ell$, recursively defined by $\ell^j=\ell \circ \ell^{j-1}$ with $\ell^0$ the identity map.
Next, we will use $\ell$ and its iterates to construct a vector $\lala(\uu)\in\R^{m,+}$ such that $\lala(\uu)\uu\in (B-v_{k^*})\cup\{0\}$.
We will consider two cases.
The first case is when $\ell^{m}(k^*)\in\{m+1,\dots,n\}$, so $v_{\ell^{m}(k^*)}\in B$.
We then let
$
\lambda_k \coloneqq |\{i\in\{0,\dots,m-1\}\colon \ell^{i}(k^*)=k\}|
$
for all $k\in\{1,\dots,n\}$ and let $\lala(\uu)\coloneqq(\lambda_1,\dots,\lambda_m)$.
Since $\lambda_k(\uu)=\lambda_k\geq0$ for all $k=\{1,\dots,m\}$ and because---since $\ell^0(k^*)=k^*$---$\lambda_{k^*}(\uu)=\lambda_{k^*}\geq1$, we have that $\lala(\uu)\in\R^{m,+}$.
Furthermore,
\[
\lala(\uu)\uu=\sum_{k=1}^{m} \lambda_k(\uu) u_k=\sum_{k=1}^m \lambda_k u_k=\sum_{k=1}^{n} \lambda_k u_k=\sum_{k=1}^n \sum^{m-1}_{\substack{i=0\\\ell^i(k^*)=k}}u_k=\sum_{i=0}^{m-1} u_{\ell^i(k^*)},
\]
where the third equality follows from the fact that $u_k=0$ for all $k\in\{m+1,\dots,n\}$.
Hence, we find that
\begin{multline*}
\lala(\uu)\uu=\sum_{i=0}^{m-1} u_{\ell^i(k^*)}=\sum_{i=0}^{m-1} (v_{\ell^{i+1}(k^*)}-v_{\ell^{i}(k^*)})\\=v_{\ell^{m}(k^*)}-v_{k^*} \in B-v_{k^*}\subseteq(B-v_{k^*})\cup\{0\}.
\end{multline*}
The second case is when $\ell^m(k^*)\in\{1,\dots,m\}$.
We then have that $\{\ell^{i}(k^*)\colon i\in\{0,\dots, m\}\}\subseteq\{1,\dots, m\}$ because of the following reasoning.
Assume \emph{ex absurdo} that there is some $i\in\{0,\dots, m-1\}$ such that $\ell^{i}(k^*)\in \{m+1,\dots,n\}$. 
Then $\ell^{i}(k^*)=\ell^{i+1}(k^*)\in \{m+1,\dots,n\}$ and thus by induction also $\ell^{m}(k^*)\in \{m+1,\dots,n\}$, contradicting the starting point of this second case.
Hence, the set $\{\ell^i(k^*)\colon i\in\{0,\dots,m\}\}$ is indeed a subset of $\{1,\dots,m\}$.
By the pigeonhole principle, there must therefore exist $j_1,j_2\in \{0,m\}$ such that $j_1\neq j_2$ and $\ell^{j_1}(k^*)=\ell^{j_2}(k^*)$.
Without loss of generality, we can assume that $j_1<j_2$.
We now let
$
\lambda_k(\uu)\coloneqq |\{i\in\{j_1,\dots,j_2-1\}\colon \ell^{i}(k^*)=k\}|
$
for all $k\in\{1,\dots,m\}$.
The resulting vector $\lala(\uu)=(\lambda_1(\uu),\dots,\lambda_m(\uu))$ then belongs to $\R^{m,+}$ because $\lambda_k(\uu)\geq0$ for all $k\in\{1,\dots,m\}$ and because---since $\ell^{j_1}(k^*)=k^*$---$\lambda_{k^*}(\uu)\geq1$.
Furthermore,
\begin{multline*}
\lala(\uu)\uu=\sum_{k=1}^{m} \lambda_k(\uu) u_k=\sum_{k=1}^m \sum_{\substack{i=j_1\\\ell^i(k^*)=k}}^{j_2-1} u_k=\sum_{i=j_1}^{j_2-1} u_{\ell^{i}(k^*)}\\=\sum_{i=j_1}^{j_2-1}(v_{\ell^{i+1}(k^*)}-v_{\ell^{i}(k^*)})= v_{\ell^{j_2}(k^*)}-v_{\ell^{j_1}(k^*)}=0\in (B-v_{k^*})\cup\{0\}.
\end{multline*}

In conclusion, for any $\uu\in \times_{k=1}^m A_k$, we have found some $\lala(\uu)\in \R^{m,+}$ such that $\lala(\uu)\uu\in(B-v_{k^*})\cup\{0\}.$
Hence, there is some $\lala\in \Lamp_m(\times_{k=1}^m A_k)$ such that $D\coloneqq \{\lala(\uu)\uu\colon \uu\in\times_{k=1}^m A_k\}\subseteq (B-v_{k^*})\cup\{0\}$.
Since $A_k\in K$ for all $k\in\{1,\dots,m\}$ and $\lala\in \Lamp_m(\times_{k=1}^m A_k)$ we also know that $D\in\Posi(K)$.
Furthermore, since $K$ is coherent, $\Posi(K)=K$ because of \cite[Proposition~24]{de2018desirability}.
Hence, $D\in K$.
Since $D\subseteq (B-v_{k^*})\cup\{0\}$, it therefore follows from $\mathrm{K}_4$ and $\mathrm{K}_0$ that $B-v_{k^*}\in K$.
\qed
\end{proof}
\begin{corollary}\label{cor:CA-u}
Consider any coherent choice function $C$ and any option set $A\in\Q$.
Then $C(A)-u\in K_C$ for any $u\in R_C(A)$.
\end{corollary}
\begin{proof}
From \cref{prop:CcohKC} we know that $K_C$ is coherent.
Consider $u\in R_C(A)$.
Then $u\in A$ but $u\notin C(A)$.
It therefore follows from axiom $\mathrm{C}_1$ that $0\notin C(A-u)=C((A-u)\cup\{0\})$.
Thus by definition of $K_C$ we have that
\(
A-u\in K_C.
\) 
Since this is true for every $u\in R_C(A)\subseteq A$, and since we know from \cref{prop:CcohKC} that $K_C$ is coherent, it follows from \cref{th:speedyAss} that for any $u\in R_C(A)$, $C(A)-u=A\setminus R_C(A)-u\in K_C$.
\qed
\end{proof}

\begin{lemma}\label{lem:consistenC}
Consider any choice function $C$ and any coherent choice function $C'$. 
Then $C'\in\mathcal{C}_C\Leftrightarrow \A_C \subseteq K_{C'}$.
\end{lemma}
\begin{proof}
First we prove the implication from the left to the right.
So assume that $C'\in \mathcal{C}_C$.
This implies that $C'(A)\subseteq C(A)$ for all $A\in\Q$.
Now take any $B\in\A_C$.
This implies that there is some $A\in\Q$ and $u\in R_C(A)$ such that $B=C(A)-u$.
Since $u\in R_C(A)$ and $C'(A)\subseteq C(A)$, we know that $u\in R_{C'}(A)$.
Since $C'$ is coherent, it therefore follows from \cref{cor:CA-u} that $C'(A)-u\in K_{C'}$
Also, $K_{C'}$ is coherent by \cref{prop:CcohKC}. 
From axiom $\mathrm{K}_4$ and $C'(A)\subseteq C(A)$ it therefore follows that $B=C(A)-u \in K_{C'}$.

Now for the other implication.
To prove that $C'\in\mathcal{C}_C$, since $C'$ is coherent, we have to prove that $C'(A)\subseteq C(A)$ for any $A\in \Q$.
Take any $A\in \Q$ and any $u\in C'(A)\subseteq A$.
Then we have that $0\in C'(A-u)=C'((A-u)\cup\{0\})$ by axiom $\mathrm{C}_1$.
It follows that $A-u\notin K_{C'}$ from the definition of $K_{C'}$.
Since $K_{C'}$ is coherent because of \cref{prop:CcohKC}, it therefore follows from axiom $\mathrm{K}_4$ that $C(A)-u\notin K_{C'}$.
Since $\A_C\subseteq K_{C'}$, this implies that $C(A)-u\notin \A_C$.
By definition of $\A_C$, then $u\notin R_C(A)$ and thus, since $u\in \A$, $u\in C(A)$ as required. 
\qed
\end{proof}


\begin{theorem}\label{th:consistEquiv}
A choice function $C$ is consistent if and only if the assessment $\A_C$ is consistent.
\end{theorem}
\begin{proof}
First we prove the implication from left to right.
Since $C$ is consistent, $\mathcal{C}_C$ is non-empty.
Consider any coherent $C'\in\mathcal{C}_C$.
By \cref{lem:consistenC}, it then follows that $\A_C\subseteq K_{C'}$.
Moreover, since $C'$ is coherent, we know from \cref{prop:CcohKC} that $K_{C'}$ is coherent as well.
Hence, $K_{C'}\in\Ka(\A_C)$.
So $\Ka(\A_C)$ is non-empty and $\A_C$ is therefore consistent.

Now the other implication.
Since $\A_C$ is consistent, there is a coherent set of desirable option sets $K$ such that $\A_C\subseteq K$.
Since $K$ is coherent, it follows from \cref{prop:KcohCK} that $C_K$ is coherent and from \cref{th:unique} that $K=K_{C_K}$ and therefore $\A_C\subseteq K= K_{C_K}$.
\cref{lem:consistenC} therefore tells us that $C_K\in\mathcal{C}_{C}$, so $\mathcal{C}_C$ is non-empty and $C$ is therefore consistent.
\qed
\end{proof}

\begin{proposition}\label{prop:natequiv}
Let $C$ be any choice function.
Then $C_{\Ex(\A_C)}(A)=\Ex(C)(A)$ for all $A\in \Q$
\end{proposition}
\begin{proof}
We have to prove for all $A\in\Q$ that $C_{\Ex(\A_C)}(A)=\Ex(C)(A)$.
Consider any $A\in\Q$.
We will prove equality by proving inclusion of each set in the other.

Firstly, take any $u\in\Ex(C)(A).$
It then follows from the definition of $\Ex(C)$ that there is some $C'\in\mathcal{C}_C$ for which $u\in C'(A)\subseteq A$.
Axiom $\mathrm{C}_1$ therefore implies that $0\in C'(A-u)=C'(((A-u)\setminus\{0\})\cup\{0\})$, so $(A-u)\setminus\{0\}\notin K_{C'}$.
Also, by \cref{lem:consistenC}, $\A_C\subseteq K_{C'}$ because $C'\in \mathcal{C}_C$.
Since $K_{C'}$ is coherent because of \cref{prop:CcohKC}, it follows from $\A_C\subseteq K_{C'}$ that $K_{C'}\in \Ka(\A_C)$.
Therefore also $\Ex(\A_C)\subseteq K_{C'}$.
From $(A-u)\setminus\{0\}\notin K_{C'}$ and $\Ex(\A_C)\subseteq K_{C'}$, it follows that $(A-u)\setminus\{0\}\notin \Ex(\A_C)$.
Since $u\in A$, \cref{eq:CK} therefore implies that $u\in C_{\Ex(\A_C)}(A)$.

Secondly, take any $u\in C_{\Ex(\A_C)}(A)$.
It then follows from \cref{eq:CK} that $u\in A$ and $(A-u)\setminus\{0\}\notin\Ex(\A_C)$.
Because of \cref{eq:defEx}, this implies that there is some $K\in\Ka(\A_C)$ such that $(A-u)\setminus\{0\}\notin K$.
Since $u\in A$, it therefore follows from \cref{eq:CK} that $u\in C_K(A)$.
Since $K\in\Ka(\A_C)$, we know that $K$ is coherent and that $\A_C\subseteq K$.
Since $K$ is coherent, it follows from \cref{prop:KcohCK} that $C_K$ is coherent as well, and from \cref{th:unique} that $K=K_{C_K}$.
Since $C_K$ is coherent and $\A_C\subseteq K=K_{C_K}$, it follows from \cref{lem:consistenC} that $C_K\in\mathcal{C}_C$.
Since $C_K\in\mathcal{C}_C$ and $u\in C_K(A)$, it follows from the definition of $\Ex(C)$ that $u\in \Ex(C)(A)$.
\qed
\end{proof}

\begin{proofof}{\cref{th:natequiv}}
That $C$ is consistent if and only if $\emptyset\notin \Ex(\A_C)$ follows directly from \cref{th:consistEquiv,th:consifandonlyifcoh}.
The second statement follows from \cref{prop:natequiv}.
\end{proofof}


\begin{proposition}\label{prop:posimp}
Let $\A=\{A_1,\ldots,A_n\} \in \QQ$ be a finite assessment, with $n\in\N$. Then for any set $P_1\in \Posi(\A)$ there is a set $P_2\in\Posi'(A_1,\dots,A_n)$ such that $P_2\subseteq P_1$.
\end{proposition}

\begin{proof}
Consider any set $P_1\in \Posi(\A)$.
It then follows from the definition of $\Posi(\A)$ that there is some $m\in\N$, a sequence $B_1,\ldots,B_m\in\A$ and a map $\boldsymbol{\mu}\in \Lamp_m(\times_{k=1}^m B_k)$ such that $P_1=\{\mumu(\uu)\uu\colon \uu\in\times_{k=1}^m B_k\}$.
Since $B_1,\ldots,B_m\in\A$, there is a function $\ell\colon\{1,\dots,m\}\to\{1,\dots,n\}$ such that $B_{j}=A_{\ell(j)}$ for all $j\in \{1,\dots,m\}$.
Let $\ell^{-1}$ denote its preimage.
Also note that $\mathbf{v}_{\uu}\coloneqq(u_{\ell(1)},\dots,u_{\ell(m)})\in \times_{k=1}^m B_k$ for all $\uu\in \times_{k=1}^k A_k$ by definition of $\ell$.
We now define 
\[\textstyle
\boldsymbol{\phi}\colon \R^{m,+}\to \R^{n,+}\colon (\mu_1,\dots,\mu_m)\mapsto \left(\sum_{j\in \ell^{-1}(1)}\mu_j,\dots,\sum_{j\in \ell^{-1}(n)}\mu_j\right),
\]
where the empty sum is zero, 
and let $\lala\in\Lamp_n(\times_{k=1}^n A_k)$ be defined by $\lala(\uu)\coloneqq\boldsymbol{\phi}(\boldsymbol{\mu}(\mathbf{v}_{\uu}))$ for all $\uu\in \times_{k=1}^n A_k$.
Take any $\uu\in \times_{k=1}^n A_k$. 
Then 
\[
\lala(\uu)\uu=\sum_{k=1}^n \lambda_k(\uu) u_k=\sum_{k=1}^n \sum_{j\in \ell^{-1}(k)}\mu_j(\mathbf{v}_{\uu}) u_k=\sum_{j=1}^m \mu_j(\mathbf{v}_{\uu}) u_{\ell(j)}=\mumu(\mathbf{v}_{\uu}) \mathbf{v}_{\uu}.
\]
Since $\mathbf{v}_{\uu}\in \times_{k=1}^m B_k$, this implies that $\lala(\uu)\uu=\boldsymbol{\mu}(\mathbf{v})\mathbf{v}\in P_1$.
Since this holds for any $\uu\in\times_{k=1}^n A_k$, we find that $P_2\coloneqq \{\boldsymbol{\lambda}(\mathbf{u}) \mathbf{u}\colon \mathbf{u}\in \times_{k=1}^n A_k\}\subseteq P_1 $.
\qed
\end{proof}

\begin{lemma}\label{lem:poslincomb}
Consider a tuple of options $\uu=(u_1,\dots,u_n) \in \V_{>0}^n$ and a tuple $\lala=(\lambda_1,\dots,\lambda_n)\in\R^{n,+}$, for some positive integer $n$. 
Then $\lala \uu>0$.
\end{lemma}
\begin{proof}
Note that $\lala \uu= \sum_{k=1}^n \lambda_k u_k$.
Then for any $x\in \X$, since $u_k>0$ and $\lambda_k\geq 0$ for all $k\in\{1,...,n\}$, it follows that $\sum_{k=1}^n \lambda_k u_k(x)\geq0$.
So by definition $\lala \uu \geq 0$.
Now we only need to prove that $\lala \uu \neq 0$.
Since $\lala\in\R^{n,+}$, we know that there is at least one $k\in\{1,\dots,n\}$ such that $\lambda_k>0$.
Let $k^*$ be any such $k$.
Since $u_{k^*}>0$, we know that there is at least one $x\in\X$ such that $u_{k^*}(x)>0$.
Let $x^*$ be any such $x$.
Then $(\lala \uu)(x^*)=\sum_{k=1}^n \lambda_k u_k(x^*)\geq \lambda_{k^*} u_{k^*}(x^*)>0$ and therefore $\lala \uu \neq 0$.
\qed
\end{proof}

\begin{theorem}\label{th:testnatextPre}
Consider the assessment $\A=\{\AAA\}$, with $\AAA\in\Q$ and $n\in\N\cup\{0\}$.
An option set $S\in \Q$ then belongs to $\Ex(\A)$ if and only if either $S \cap \V_{\succc0}\neq \emptyset$ or, $n\neq0$ and there is a set $T\in\Posi'(\AAA)$ such that for every $t\in T$ there is some $s \in S\cup\{0\}$ such that $t \leq s$.
\end{theorem}
\begin{proof}
We start from \cref{cor:natext}, which states that $\Ex(\A)=\Rs(\Posi(\V^s_{\succc0} \cup \A))$.
Taking into account the definition of $\Rs$, we therefore have to prove that
\begin{multline*}
\Big(\big(\exists B\in\Posi(\V^s_{\succc0} \cup \A)\big) B\setminus \V_{\precceq0} \subseteq S\Big)\Leftrightarrow \Big((S\cap\V_{\succc0}\neq\emptyset) \vee\\
\big((n\neq 0)\wedge (\exists T \in \Posi'(\AAA))(\forall t\in T)(\exists s \in S\cup\{0\})\; t\leq s\big)\Big).
\end{multline*}
We consider two cases: $S\cap \V_{>0}\neq \emptyset$ and $S\cap \V_{>0}=\emptyset$.
In the case $S\cap \V_{>0}\neq \emptyset$, the equivalence holds because both sides are true.
The righthandside is trivially true.
To show that the lefthandside is also true, we consider any $u\in S\cap \V_{>0}$ and let $B\coloneqq\{u\}.$
In the definition of $\Posi(\V_{>0}^s\cup \A)$, we can then choose $m=1$, $A_1=\{u\}\in\V_{>0}^s\cup \A$ and let $\lala\in\Lamp_1(A_1)$ be defined by $\lala(u)=1$ to see that $B=\{u\}\in\Posi(\V_{>0}^s\cup \A)$.
Furthermore, since $u\in S\cap \V_{>0}$, we clearly also have that $B\setminus \V_{\leq0}=\{u\}\subseteq S$.

So let us now consider the case $S\cap \V_{>0}=\emptyset$.
First we prove the implication to the right.
Consider any $B\in\Posi(\V^s_{>0}\cup \A)$ such that $B\setminus \V_{\leq 0} \subseteq S$.
By the definition of the $\Posi$ operator there is a positive integer $m$, a sequence $B_1,\dots,B_m\in (\V^s_{\succc0} \cup \A)$ and some $\lala \in \Lambda_{m}(\times_{k=1}^{m} B_k)$ such that $B=\{\lala(\uu)\uu \colon \uu\in \times_{k=1}^{m} B_k\}$.

We first prove that the sequence $B_1,\dots,B_{m}$ must contain at least one set in $\A$.
Assume \emph{ex absurdo} that $B_k\in\V^s_{>0}$ for all $k\in\{1,\dots,m\}$.
For all $k\in\{1,\dots,m\}$ this implies that there is some $u_k\in\V_{>0}$ such that $B_k=\{u_k\}$.
Clearly, $\times_{k=1}^m B_k$ then consists of a single tuple $\uu=(u_1,\dots,u_{m})$.
Furthermore, since $u_k\in\V_{>0}$ for all $k\in\{1,\dots,m\}$, and since $\lala(\uu)\in\R^{m,+}$, it follows from \cref{lem:poslincomb} that also $\lala(\uu)\uu=\sum_{k=1}^m \lambda_k(\uu)u_k \in \V_{>0}$.
Hence, $\lala(\uu)\uu\in B\setminus \V_{\leq0}\subseteq S$, contradicting the fact that $S\cap \V_{>0}=\emptyset$. So there is indeed at least one $k\in\{1,\ldots,m\}$ such that $B_k\in\A$.
This already implies that $n\neq 0$.
It also implies, without loss of generality, that there is some positive integer $j\in\{1,\ldots,m\}$ such that $B_1,\dots,B_j\in\A$ and that there are $p_k\in\V_{>0}$ such that $B_k=\{p_k\}\in \V_{>0}^s$ for all $k\in \{j+1,...,m\}$, where the second sequence is empty if $j=m$.

We will now prove, for any $\uu\in \times_{k=1}^{m} B_k$, that at least one of $\lambda_1(\uu),\dots,\lambda_j(\uu)$ must be non-zero.
Assume \emph{ex absurdo} that they are all zero.
Then it cannot be that $j=m$, as then $\lala(\uu)\notin \R^{m,+}$.
So consider now the case $j<m$.
Then, since $p_k\in\V_{>0}$ for all $k\in\{j+1,\dots,m\}$ and $\lala(\uu)=(0,\dots,0,\lambda_{j+1}(\uu),\dots,\lambda_m(\uu))\in\R^{m,+}$, it follows from \cref{lem:poslincomb} that $\lala(\uu)\uu=\sum_{k=j+1}^m \lambda_{k}(\uu)p_{k}>0$.
Hence, $\lala(\uu)\uu\in B\setminus \V_{\leq0}\subseteq S$, contradicting the fact that $S\cap \V_{>0}=\emptyset$.
So we conclude that, indeed, at least one of $\lambda_1(\uu),\dots,\lambda_j(\uu)$ must be non-zero.

Next for any $\mathbf{v}\in \times_{k=1}^j B_k$ we let $\boldsymbol{\mu}(\mathbf{v})\coloneqq(\lambda_1(\uu_{\vv}),\dots,\lambda_j(\uu_\vv))$, where $u_{\vv}\coloneqq(v_1,\dots,v_j,\allowbreak p_{j+1},\dots,p_m)\in\times_{k=1}^m B_k$.
Then $\mumu(\vv)\in\R^{j,+}$ because $\lala(\uu_{\vv})\in\R^{n,+}$ and because, as we proved above, $\uu_{\vv}\in \times_{k=1}^n B_k$ implies that $\lambda_1(\uu_{\vv}),\dots,\lambda_j(\uu_{\vv})$ contains at least one element that is non-zero.
The resulting operator $\mumu$ on $\smash{\times_{k=1}^j B_k}$ is therefore an element of $\smash{\Lamp_j(\times_{k=1}^j B_k)}$.

Consider now the set 
\(
T'\coloneqq\{\boldsymbol{\mu}(\mathbf{v})\mathbf{v}\colon\mathbf{v}\in \times_{k=1}^j B_k\}\in\Posi(\A).
\)
Since $n\neq 0$, it then follows from \cref{prop:posimp} that there is a set $T\in \Posi'(A_1,\dots,A_n)$ such that $T\subseteq T'$.
Consider now any option $t$ in $T\subseteq T'$.
Since $t\in T'$, there is some $\mathbf{v}\in \times_{k=1}^j B_k$ such that $t=\boldsymbol{\mu}(\mathbf{v})\mathbf{v}$.
Then 
\[
t=\boldsymbol{\mu}(\mathbf{v})\mathbf{v}=\sum_{k=1}^j \mu_k(\mathbf{v}) v_k \leq
\sum_{k=1}^j \mu_k(\mathbf{v}) v_k+\sum_{k=j+1}^i \lambda_k(\uu_{\vv}) p_k=\lala(\uu_{\vv}) \uu_{\vv}.
\]
Furthermore, since $\lambda(\uu_{\vv})\uu_{\vv}\in B$ and $B\setminus \V_{\leq0}\subseteq S$, we also know that either $\lala(\uu_{\vv})\uu_{\vv}\leq 0$ or $\lala(\uu_{\vv})\uu_{\vv}\in S$. 
Since $t \leq \lala(\uu_{\vv})\uu_{\vv}$, we can combine this to imply that $(\exists s\in S\cup\{0\})\; t\leq s$.
So $(\exists s\in S\cup\{0\})\; t\leq s$ is true for any $t\in T$.

Now we prove the implication from the right to the left.
Since $S\cap\V_{\succc0}=\emptyset$, the second condition must be true.
That is, $n\neq0$ and there is a set $T\in \Posi'(\AAA)$ such that for all $t\in T$ there exists an $s\in S\cup\{0\}$ such that $t\leq s$.
Since $T\in \Posi'(\AAA)$, there is some function $\lala\in \Lamp_n(\times_{k=1}^n A_k)$ such that $T=\{\lala(\uu)\uu\colon\uu\in\times_{k=1}^n A_k\}$.
Let $h\colon T\to S\cup\{0\}$ be a function that maps each element $t$ of $T$ to a corresponding element $h(t)\in S\cup\{0\}$ such that $t\leq h(t)$. 
Now consider the set $H\coloneqq\{h(t)-t\colon t\in T\}$.
By the definition of $h$, we have that $h(t)\geq t$ for all $t\in T$. 
Hence, for every element $g\in H$, we have that $g\geq0$.
So if we let $H_0\coloneqq H\setminus\{0\}$, then $H_0\subseteq \V_{>0}$.
Since $A_k\in\Q$ is finite for every $k\in\{1,\dots,n\}$, we know that $|\times_{k=1}^n A_k|$ is finite.
From this it follows that $T$ is finite, which in turn implies that $H$ and $H_0$ are finite.
So without loss of generality, we have that $H_0=\{h_1,\dots,h_j\}$ with $j\in\N\cup\{0\}$.
Consider the sequence $A_1,\dots,A_n,\{h_1\},\dots,\{h_j\}$ in $\A\cup\V_{>0}^s$ and, for all $\vv\in(\times_{k=1}^n A_k)\times(\times_{k=1}^j \{h_j\})$, let $\mumu(\vv)\coloneqq(\mu_1(\vv),\dots,\mu_{n+j}(\vv))$, where for all $k\in\{1,\dots,n+j\}$,
\[
\mu_k(\vv)\coloneqq
\begin{cases}
\lambda_k(\uu_\vv)& \text{if }k\leq n,\\
1 & \text{if } k>n \text{ and } h_{k-n}=h(\lala(\uu_\vv)\uu_\vv)-\lala(\uu_\vv)\uu_\vv,\\
0 & \text{otherwise,}
\end{cases}
\]
with $\uu_\vv\coloneqq(v_1,\dots,v_n)\in \times_{k=1}^n A_k$.
Then $\mumu(\vv)\in \R^{n+j,+}$ because $\lala(\uu_{\vv})\in\R^{n,+}$.
The resulting operator $\mumu$ on $(\times_{k=1}^n A_k)\times(\times_{k=1}^j \{h_j\})$ is therefore an element of $\Lamp_{n+j}((\times_{k=1}^n A_k)\times(\times_{k=1}^j \{h_j\}))$.
Since $A_1,\dots,A_n,\{h_1\},\dots,\{h_j\}\in\A\cup\V_{>0}^s$, this implies that $B\coloneqq\{\mumu(\vv)\vv \colon \vv\in (\times_{k=1}^n A_k)\times(\times_{k=1}^j \{h_j\})\}\in \Posi(\A\cup\V_{>0}^s)$.
Consider now any element $b\in B$.
Then there is some $\vv\in (\times_{k=1}^n A_k)\times(\times_{k=1}^j \{h_j\})$ such that $b=\mumu(\vv)\vv$.
Since $\uu_{\vv}\in\times_{k=1}^n A_k$, we know that $\lala(\uu_{\vv})\uu_{\vv}\in T$ and therefore, that $h_{\vv}\coloneqq h(\lala(\uu_\vv)\uu_\vv)-\lala(\uu_\vv)\uu_\vv\in H$.
We now consider two cases: $h_{\vv}=0$ and $h_{\vv}\in H_{0}$.
If $h_{\vv}=0$, then $\mu_k(\vv)=0$ for all $k\in\{n+1,\dots,n+j\}$ and therefore $\sum_{k=n+1}^{n+j}\mu_k(\vv)v_k=0=h_v$.
If $h_{\vv}\in H_0$, then there is exactly one $k^*\in\{n+1,\dots,n+j\}$ such that $h_{\vv}=h_{k^*-n}$, so $\mu_{k^*}(\vv)=1$ and $\mu_k(\vv)=0$ for all $k\in\{n+1,\dots,n+j\}\setminus\{k^*\}$, and therefore $\sum_{k=n+1}^{n+j} \mu_k(\vv)v_k=v_{k^*}=h_{k^*-n}=h_{\vv}$.
Hence, in both cases, $\sum_{k=n+1}^{n+j}\mu_k(\vv) v_k=h_{\vv}$, which implies that
\begin{multline*}
b=\mumu(\vv)\vv=\sum_{k=1}^{n+j}\mu_k(\vv) v_k=\sum_{k=1}^{n}\lambda_k(\uu_\vv) v_k+h_{\vv}\\=\lala(\uu_\vv)\uu_\vv+h(\lala(\uu_\vv)\uu_\vv)-\lala(\uu_\vv)\uu_\vv=h(\lala(\uu_\vv)\uu_\vv)\in S\cup\{0\}.
\end{multline*}
Since this is true for every $b\in B$, it follows that $B\subseteq S\cup\{0\}$ and therefore that $B\setminus \V_{\leq0}\subseteq S$.
\qed
\end{proof}

\begin{proofof}{\cref{th:testnatext}}
From \cref{th:testnatextPre} it follows that we only have to prove, for $n\neq 0$, the following equivalence:
\begin{multline*}
\big((\exists T\in \Posi'(\AAA))(\forall t\in T)(\exists s\in S\cup\{0\})\;t\leq s\big)\\ \Leftrightarrow \big((\forall \uu\in\times_{j=1}^n A_j) (\exists \lala\in \R^{n,+})( \exists s\in S\cup\{0\})\; \lala \uu \leq s\big).
\end{multline*}
For the implication to the right, assume that there is some $T\in \Posi'(\AAA)$ for which the condition holds.
By definition of $\Posi'(\AAA)$, there is a corresponding function $\mumu\in\Lamp_n(\times_{j=1}^n A_j)$ such that $T=\{\mumu(\uu) \uu \colon \uu \in \times_{j=1}^n A_j\}$.
Consider any $\uu \in \times_{j=1}^n A_j$. Then there is some $t\in T$ such that $t=\mumu(\uu)\uu$.
Also, since $T$ satisfies the condition on the left, there is some $s\in S\cup\{0\}$ such that $\mumu(\uu)\uu=t\leq s$.
So if we let $\lala\coloneqq\mumu(\uu)$ then the righthandside follows.

For the implication to the left, assume that, for any $\uu\in \times_{j=1}^n A_j$, we have some $\lala_\uu\in \R^{n,+}$ and $s_\uu \in S\cup\{0\}$ such that $\lala_\uu \uu \leq s_{\uu}$.
Now let $\mumu\in \Lamp_n(\times_{j=1}^n A_j)$ be defined by $\mumu(\uu)\coloneqq\lala_\uu$ for any $\uu\in \times_{j=1}^n A_j$.
Then we have that $T\coloneqq\{\mumu(\uu)\uu\colon\uu\in\times_{j=1}^n A_j\}\in \Posi'(\AAA)$.
Furthermore, for any $t\in T$, there is a corresponding $\uu\in\times_{j=1}^n A_j$ such that $t=\mumu(\uu)\uu=\lala_\uu \uu \leq s_\uu\in S\cup\{0\}$.
\end{proofof}




\begin{proofof}{\cref{prop:simpindiv}}
Since $\A_1$ and $\A_2$ are equivalent, we have by definition of equivalence that $\Ka(\A_1)=\Ka(\A_2)$.
\begin{multline*}
\Ka((\A \setminus \A_1) \cup \A_2)=\{K\in\Ka\colon(\A \setminus \A_1)\cup \A_2 \subseteq K\}\\=\{K\in\Ka\colon(\A \setminus \A_1)\subseteq K, \A_2 \subseteq K\}\\=\{K\in\Ka\colon(\A \setminus \A_1)\subseteq K\}\cap\{K\in\Ka\colon \A_2 \subseteq K\}
\end{multline*}
and since $\{K\in\Ka\colon \A_2 \subseteq K\}=\Ka(\A_2)=\Ka(\A_1)=\{K\in\Ka\colon \A_1 \subseteq K\}$, it follows that 
\begin{multline*}
\Ka((\A \setminus \A_1) \cup \A_2)=\{K\in\Ka\colon(\A \setminus \A_1)\subseteq K\}\cap\{K\in\Ka\colon \A_1 \subseteq K\}\\=\{K\in\Ka\colon(\A \setminus \A_1)\cup \A_1 \subseteq K\}.
\end{multline*}
Since $\A_1\subseteq \A$, we have that $(\A \setminus \A_1)\cup \A_1=\A$ and thus $\Ka((\A \setminus \A_1) \cup \A_2)=\{K\in\Ka\colon \A\subseteq K\}=\Ka(\A)$.
\end{proofof}

\begin{proofof}{\cref{prop:wegMetDieNul}}
For equivalence of the assessments it suffices, by definition of equivalence, to show that $A\in K \Leftrightarrow A\setminus\{0\}\in K$.
Any coherent $K$ that contains a set $A\in \Q$ will also contain $A\setminus \{0\}$ by axiom $\mathrm{K}_0$ and vice versa by axiom $\mathrm{K}_4$.
\end{proofof}

\begin{proofof}{\cref{th:simp1}}
By the definition of equivalence of these two assessments it suffices to prove for any coherent set of desirable option sets $K$ that
\[A \in K \Leftrightarrow (A\setminus \{u\}) \in K.\]
The implication to the left follows immediately from axiom $\mathrm{K}_4$.
So it remains to prove the implication to the right.
Consider any $A\in K$.
We will consider two cases: $\mu>0$ and $\mu=0$.

First we consider the case where $\mu>0$.
Since $0\leq \mu v -  u$ either $\mu v -  u =0$ or $(\mu v -  u) \in\V_{>0}$.
If $(\mu v -  u) \in\V_{>0}$, then let $p\coloneqq \mu v -  u$, and otherwise let $p$ be any option in $\V_{>0}$.
Let
\[
\lala\colon A\times \{p\}\to\R^{2,+}\colon (w,t)\mapsto 
\begin{cases}
(\mu^{-1},0) &\text{if } w=u\text{ and }\mu v -  u =0\\
(\mu^{-1},\mu^{-1}) &\text{if } w=u\text{ and } \mu v - u > 0\\
(1,0) & \text{otherwise.} 
\end{cases}
\]
Then, 
\begin{multline*}
\lala(u,p) (u,p)=\\
\begin{cases}
\mu^{-1} u + 0\cdot p & \text{if }\mu v-u=0\\
\mu^{-1} u + \mu^{-1} p &\text{if }\mu v-u>0
\end{cases}
=
\begin{cases}
\mu^{-1} \mu v & \text{if }\mu v-u=0\\
\mu^{-1} u + \mu^{-1} (\mu v-u) &\text{if }\mu v-u>0
\end{cases}\\
=v
\end{multline*}
and $\lala(w,p)(w,p)=w$ for all $w\in A\setminus \{u\}$.
Hence,
\[
\{\lala(\uu)\uu\colon \uu \in A\times \{p\}\}=(A\setminus\{u\})\cup\{v\}=A\setminus\{u\},
\]
where the last equality follows because $v\in A$ and $v\neq u$.
Since $A\in K$ by assumption and $\{p\}\in \V_{>0}^s\subseteq K$ because of $\mathrm{K}_2$, it therefore follows from $\mathrm{K}_3$ that $A\setminus\{u\}\in K$.

Next we consider the case $\mu=0$.
Then $u\leq\mu v=0$.
We again consider two cases: $u=0$ and $u\neq 0$.
If $u=0$, it follows from $\mathrm{K}_1$ that $A\setminus\{u\}=A\setminus\{0\}\in K$.
So it remains to consider the case $u\neq0$.
Since $u\leq 0$, we then have that $u<0$, so $-u>0$ and therefore $p\coloneqq -u\in \V_{>0}$.
Let
\[
\lala\colon A\times \{p\}\to\R^{2,+}\colon (w,t)\mapsto 
\begin{cases}
(1,1) & \text{if }w=u\\
(1,0) & \text{otherwise.} 
\end{cases}
\]
Then,
$\lala(u,p) (u,p)= u +p=u-u= 0$ and $\lala(w,p)(w,p)=w$ for all $w\in A\setminus \{u\}$.
Hence,
\[
\{\lala(\uu)\uu\colon \uu \in A\times \{p\}\}=(A\setminus\{u\})\cup\{0\}.
\]
Since $A\in K$ by assumption and $\{p\}\in \V_{>0}^s\subseteq K$ because of $\mathrm{K}_2$, it therefore follows from $\mathrm{K}_3$ that $(A\setminus\{u\})\cup\{0\}\in K$.
Axiom $\mathrm{K}_1$ therefore implies that $(A\setminus\{u\})\setminus\{0\}=((A\setminus\{u\})\cup\{0\})\setminus\{0\}\in K$, and so it follows from $\mathrm{K}_4$ that $A\setminus\{u\}\in K$.
\end{proofof}

To facilitate the statement of our next three results, we introduce the following notation for the second and fourth quadrant of $\R^2$, excluding the origin: \[Q_2\coloneqq\{(x,y)\in \R^2\setminus \{(0,0)\}\colon x\leq 0,y\geq0\}\] and \[Q_4\coloneqq\{(x,y)\in \R^2\setminus\{(0,0)\} \colon x\geq 0,y\leq0\}.\]

\begin{lemma}\label{lem:Q4}
Consider any $u=(u_1,u_2)$ and $v=(v_1,v_2)$ in $Q_2\cup Q_4$ such that $u\not\leq \mu_{v}v$ for all $\mu_{v}\geq0$ and $v\not\leq\mu_{u} u$ for all $\mu_u\geq 0$. Then $\{u,v\}\cap Q_4\neq \emptyset$.
\end{lemma}
\begin{proof}
Assume \emph{ex absurdo} that $u,v\notin Q_4$.
Since $u,v\in Q_2\cup Q_4$ and $Q_2\cap Q_4=\emptyset$, this implies that $u,v\in Q_2$, so $u_1\leq 0$, $u_2 \geq 0$,$v_1\leq 0$,$v_2 \geq 0$, $(u_1,u_2)\neq (0,0)$ and $(v_1,v_2)\neq (0,0)$.
Assume without loss of generality that $u_1 v_2 \leq u_2 v_1$ (if not, then simply reverse the roles of $u$ and $v$).
We consider two cases $v_2> 0$ and $v_2=0$.

If $v_2>0$, it follows from $u_1 v_2 \leq u_2 v_1$ that $u_1\leq \frac{u_2}{v_2} v_1$.
Hence, if we let $\mu_v\coloneqq\frac{u_2}{v_2}\geq 0$, then $\mu_v v=\frac{u_2}{v_2}(v_1,v_2)=(\frac{u_2 v_1}{v_2}, u_2)\geq (u_1, u_2)=u$; a contradiction.

So it remains to consider the case $v_2=0$.
In that case, it follows from $u_1 v_2 \leq u_2 v_1$ that $0\leq u_2 v_1$ and from $(v_1,v_2)\neq(0,0)$ and $v_1\leq 0$ that $v_1<0$.
Since $u_2 \geq 0$, $v_1< 0$ and $0\leq u_2 v_1$, we see that $u_2=0$.
So $u=(u_1,0)$ and $v=(v_1,0)$.
But then either $u\leq v$ or $v\leq u$.
So if we let $\mu_u=\mu_v=1$, then either
\(
 u\leq \mu_v v\) or \(
 v\leq \mu_u u
\); again a contradiction.
\end{proof}

\begin{lemma}\label{lem:Q2enQ4}
Consider any $u=(u_1,u_2)$ and $v=(v_1,v_2)$ in $\R^2\setminus\{(0,0)\}$ such that \(
 u\not\leq \mu_v v\) for all $\mu_v\geq0$ and \(
 v\not\leq \mu_u u
\) 
for all $\mu_u\geq 0$.
Then either $u\in Q_2$ and $v\in Q_4$, or, $u\in Q_4$ and $v\in Q_2$.
\end{lemma}
\begin{proof}
We first prove that $u\in Q_2\cup Q_4$.
Assume \emph{ex absurdo} that $u\notin Q_2\cup Q_4$.
Since $u\neq(0,0)$, there are two possibilities: $u_1>0$ and $u_2>0$, or, $u_1<0$ and $u_2<0$.
If $u_1>0$ and $u_2>0$, then for $\mu_u\coloneqq\max(\frac{|v_1|}{u_1},\frac{|v_2|}{u_2})$, we see that $v\leq \mu_u u$; a contradiction.
If $u_1<0$ and $u_2<0$, then for $\mu_v\coloneqq0$ we see that $u\leq \mu_v v$; again a contradiction.
Hence, indeed, $u\in Q_2\cup Q_4$.
A completely analogous argument---just swap the roles of $u$ and $v$---yields that also $v\in Q_2\cup Q_4$.
So it follows from \cref{lem:Q4} that $\{u,v\}\cap Q_4\neq \emptyset$.

Next, we prove that $\{u,v\}\cap Q_2\neq \emptyset$.
Consider the vectors $u'\coloneqq(u_2,u_1)$ and $v'\coloneqq(v_2,v_1)$.
Then since $u\not\leq \mu_v  v$ for all $\mu_v\geq 0$, we see that also $u'\not\leq \mu_{v'} v'$ for all $\mu_{v'}\geq 0$.
Secondly, since $v\not\leq \mu_u  u$ for all $\mu_u\geq 0$, we see that also $v'\not\leq \mu_{u'} u'$ for all $\mu_{u'}\geq 0$.
It therefore follows from \cref{lem:Q4} that, indeed, $\{u,v\}\cap Q_2\neq \emptyset$.
So we conclude that $\{u,v\}\cap Q_4\neq \emptyset$ and $\{u,v\}\cap Q_2\neq \emptyset$.
Since $Q_2\cap Q_4=\emptyset$, this implies that either $u\in Q_2$ and $v\in Q_4$, or, $u\in Q_4$ and $v\in Q_2$. 
\qed
\end{proof}

\begin{proofof}{\cref{prop:simp1inR2}}
If $A$ contains at most two options, then we can choose $B=A$.
Otherwise, assume \emph{ex absurdo} that $A$ consists of $n\geq 3$ options and cannot be simplified any further by \cref{th:simp1}.
Since $A$ then consists of at least three options, we can fix three distinct options in $A$: $u,v$ and $w$.
We consider two cases: when at least one these three options equals $(0,0)$ and when none of them equals $(0,0)$.

In the first case, we can assume without loss of generality that $u=(0,0)$.
But then we can remove $u$ by \cref{th:simp1} because $u\neq v$ and, for $\mu=0$, $u\leq \mu v$.
This contradicts our assumption that $A$ cannot be simplified any further.

So it remains to consider the second case, where none of the options $u,v,w$ equals $(0,0)$.
Since $A$ cannot be simplified any further by \cref{th:simp1}, and since $u\neq v$, it must be that $u\not\leq \mu_v v$ for all $\mu_v\geq 0$ and $v\not\leq \mu_u u$ for all $\mu_u\geq 0$.
Since $u\neq(0,0)$ and $v\neq (0,0)$, it therefore follows from Lemma 24 that
\begin{equation}\label{eq:1}
(u\in Q_2\text{ and }v\in Q_4)\text{ or }(u\in Q_4\text{ and }v\in Q_2).
\end{equation}
A completely analogous argument yields that
\begin{equation}\label{eq:2}
(u\in Q_2\text{ and }w\in Q_4)\text{ or }(u\in Q_4\text{ and }w\in Q_2)
\end{equation}
and that
\begin{equation}\label{eq:3}
(v\in Q_2\text{ and }w\in Q_4)\text{ or }(v\in Q_4\text{ and }w\in Q_2).
\end{equation}
If $u\in Q_2$, then since $Q_2\cap Q_4=\emptyset$, it follows from \cref{eq:1} that $v\in Q_4$ and from \cref{eq:2} that $w\in Q_4$, contradicting \cref{eq:3}.
If $u\in Q_4$, then since $Q_2\cap Q_4=\emptyset$, it follows from \cref{eq:1} that $v\in Q_2$ and from \cref{eq:2} that $w\in Q_2$, again contradicting \cref{eq:3}.
\end{proofof}

\begin{proofof}{\cref{th:simp2}}
We prove the two inclusions $\Ka(\A)\subseteq\Ka(\A\setminus\{A\})$ and $\Ka(\A\setminus\{A\})\subseteq\Ka(\A)$ separately.
$\Ka(\A)\subseteq\Ka(\A\setminus\{A\})$ follows directly from the definition.
Now we prove the other inclusion.
Consider any $K\in\Ka(\A\setminus\{A\})$.
Then on the one hand, since $A\in\Ex(\A\setminus\{A\})$, it follows from \cref{eq:defEx} that $A\in K$.
On the other hand, $K\in\Ka(\A\setminus\{A\})$ also implies that $K\in \Ka$ and $\A\setminus\{A\}\subseteq K$.
Since $\A\setminus\{A\}\subseteq K$ and $A\in K$, we find that $\A \subseteq K$.
Since $K\in\Ka$, this implies that $K\in\Ka(\A)$.
Since this is true for every $K\in\Ka(\A\setminus\{A\})$, it follows that $\Ka(\A\setminus\{A\})\subseteq \Ka(\A)$.
\end{proofof}


\begin{proofof}{\cref{prop:hunt}}
\textsc{IsFeasible}$(\uu,s)$=True is by definition equivalent to
$(\exists \lala\in\R^{n,+})\lala \uu \leq s$.
As such it suffices to prove that
\begin{multline*}
(\exists \lala\in\R^{n,+})\lala \uu \leq s\\
\Leftrightarrow (\exists \mumu \in \R^{n+1})\big((\textstyle\sum_{j=1}^n \mu_j u_j \leq \mu_{n+1} s)\wedge (\mu_{n+1}\geq 1)\\\wedge((\forall j\in \{1,\dots,n\})\mu_j\geq0)\wedge (\textstyle\sum_{j=1}^n \mu_j\geq 1)\big).
\end{multline*}
First we prove the implication to the right.
Choose 
\[
\mumu=
\begin{cases}
(\lambda_1, \dots , \lambda_n,1) &\text{if }\sum_{j=1}^n \lambda_j\geq 1,\\
\left(\frac{\lambda_1}{\sum_{j=1}^n \lambda_j}, \dots , \frac{\lambda_n}{\sum_{j=1}^n \lambda_j},\frac{1}{\sum_{j=1}^n \lambda_j}\right) &\text{otherwise,}
\end{cases}
\] then the righthandside is true.
Next we prove the implication to the left.
Choose $\lala=(\frac{\mu_1}{\mu_{n+1}},\dots,\frac{\mu_n}{\mu_{n+1}} )$, then the lefthandside is true.
\end{proofof}

}{}

\end{document}